\documentclass[10pt]{article} % For LaTeX2e
\usepackage[preprint]{tmlr}

\usepackage[hidelinks]{hyperref}
\usepackage{url}

\usepackage{caption}
\usepackage{subcaption}
\usepackage{amsfonts}
\usepackage{amsmath,mathtools}
\usepackage{amsthm}
\usepackage{amssymb}
\usepackage{algorithm}

\usepackage{array}
\usepackage{algorithm}
\usepackage{etoolbox}

\let\classAND\AND
\let\AND\relax
\usepackage{algorithmic}
\let\AND\classAND
\AtBeginEnvironment{algorithmic}{\let\AND\algoAND}

\floatname{algorithm}{Algorithm}

\usepackage{color}

\newtheorem{theorem}{Theorem}

\newcommand{\dataset}{\mathcal{D}}
\newcommand{\cmodel}{\mathcal{C}}
\newcommand{\numensembles}{n_{\text{ensembles}}}
\newcommand{\numruns}{n_{\text{runs}}}
\newcommand{\numsteps}{n_{\text{steps}}}
\newcommand{\numcutoff}{n_{\text{cutoff}}}
\newcommand{\numdata}{{n_{\text{data}}}}

\makeatletter
\newcommand{\printfnsymbol}[1]{%
	\textsuperscript{\@fnsymbol{#1}}%
}
\makeatother

\title{No More Pesky Hyperparameters: \\Offline Hyperparameter Tuning for RL}

\author{%
	\begin{tabular}{c c c c c} 
		\name Han Wang  \thanks{Computing Science, Alberta Machine Intelligence Institute (Amii), University of Alberta, Edmonton, Alberta, Canada.} \thanks{These authors contributed equally to this work.} & 
		\name Archit Sakhadeo \printfnsymbol{1}\printfnsymbol{2} & 
		\name Adam White \printfnsymbol{1} & 
		\name James Bell \printfnsymbol{1} & 
		\name Vincent Liu \printfnsymbol{1} \\ 
		\email han8@ualberta.ca & 
		\email sakhadeo@ualberta.ca & 
		\email amw8@ualberta.ca & 
		\email jbell1@ualberta.ca & 
		\email vliu1@ualberta.ca \end{tabular} \AND
	\begin{tabular}{c c c c c} 
		\name Xutong Zhao \printfnsymbol{1} & 
		\name Puer Liu \printfnsymbol{1} & 
		\name Tadashi Kozuno \printfnsymbol{1} & 
		\name Alona Fyshe \printfnsymbol{1} & 
		\name Martha White \printfnsymbol{1} \\ 
		\email xutong@ualberta.ca & 
		\email puer@ualberta.ca & 
		\email kozuno@ualberta.ca & 
		\email alona@ualberta.ca & 
		\email whitem@ualberta.ca \end{tabular} 
}
\def\month{MM}  % Insert correct month for camera-ready version
\def\year{YYYY} % Insert correct year for camera-ready version
\def\openreview{\url{https://openreview.net/forum?id=XXXX}} % Insert correct link to OpenReview for camera-ready version

\begin{document}

\maketitle

\begin{abstract}
The performance of reinforcement learning (RL) agents is sensitive to the choice of hyperparameters. In real-world settings like robotics or industrial control systems, however, testing different hyperparameter configurations directly on the environment can be financially prohibitive, dangerous, or time consuming. We propose a new approach to tune hyperparameters from offline logs of data, to fully specify the hyperparameters for an  RL agent that learns online in the real world. The approach is conceptually simple: we first learn a model of the environment from the offline data, which we call a calibration model, and then simulate learning in the calibration model to identify promising hyperparameters. We identify several criteria to make this strategy effective, and develop an approach that satisfies these criteria. We empirically investigate the method in a variety of settings to identify when it is effective and when it fails.
\end{abstract}

\section{Introduction}\label{sec_intro}

Reinforcement learning (RL) agents are extremely sensitive to the choice of hyperparameters that regulate speed of learning, exploration, degree of bootstrapping, amount of replay and so on. The vast majority of work RL is focused on new algorithmic ideas and improving performance---in both cases assuming near-optimal hyperparameters. Indeed the vast majority of empirical comparisons involve well-tuned implementations or reporting the best performance after a hyperparameter sweep. Although progress has been made towards eliminating the need for tuning with adaptive methods \citep{white2016greedy,xu2018metagradient,mann2016adaptive,zahavy2020self,jacobsen2019meta,kingma2014adam,papini2019smoothing}, widely used agents employ dozens of hyperparameters and tuning is critical to their success \citep{henderson2018deep}.% with less focus on how to remove or reduce hyperparameter specification for these algorithms. Instead, simulation problems are used to test the ideas, to identify promising directions and carefully analyze algorithms. This research in simulation is critical for this reason, but allows the problem of hyperparameter selection to be avoided by sweeping over different hyperparameter settings.

The reason domain specialization and hyperparameter sweeps are possible---and perhaps why our algorithms are so dependent on them---is because most empirical work in RL is conducted in simulation. Simulators are critical for research because they facilitate rapid prototyping of ideas and extensive analysis. On the other hand, simulators allow us to rely on features of simulation not possible in the real world, such as exhaustively sweeping different hyperparameters. Often, it is not acceptable to test poor hyperparameters on a real system that could cause serious failures. In many cases, interaction with the real system is limited, or in more extreme cases, only data collected from a human operator is available.
% MARTHAC: Maybe this extra sentence is not needed
%Experiments on physical systems often require the researcher to babysit the robot; it is simply not feasible to evaluate dozens of settings many times. 
Recent experiments confirm significant sensitivity to hyperparameters is exhibited on real robots as well~\citep{mahmood2018benchmarking}. It is not surprising that one of the major roadblocks to applied RL is extreme hyperparameter sensitivity.  

Fortunately, there is an alternative to evaluating algorithms on the real system: using previously logged data under an existing controller (human or otherwise). This offline data provides some information about the system, which could be used to evaluate and select hyperparameters without interacting with the real world. Hyperparameters are general, and can even include a policy initialization that is adjusted online. We call this problem \emph{Data2Online}.

This problem setting differs from the standard offline or batch RL setting because the goal is to select \emph{hyperparameters} offline for the agent to use to {\em learn online in deployment}, as opposed to learning a \emph{policy} offline. Typically in offline RL a policy is learned on the batch of data, using a method like Fitted Q Iteration (FQI) \citep{ernst2005treebased,riedmiller2005neural,farahmand2009regularized}, and the resulting fixed policy is deployed. Our setting is less stringent, as the policy is learned and continually adapts during deployment. Intuitively, selecting just the hyperparameters for further online learning should not suffer from the same hardness problems as offline RL (see \citep{wang2020what} for hardness results), because the agent has the opportunity to gather more data online and adjust its policy. Even in the offline batch RL setting the hyperparameters of the learner must be set and most approaches either unrealistically use the real environment to do so \citep{wu2019behavior} or use the action values learned from the batch to choose amongst settings \citep{paine2020hyperparameter} with mixed success.

We propose a novel strategy to use offline data for selecting hyperparameters. The idea is simple: we use the data to learn a \emph{calibration model}, and evaluate hyperparameters in the calibration model. 
Learning online in the calibration model mimics learning in the environment, and so should identify hyperparameters that are effective for online learning performance. The calibration model need not be a perfect simulator to be useful for identifying reasonable hyperparameters, whereas learning a transferable policy typically requires accurate models. 

%\begin{wrapfigure}[8]{l}{0.65\textwidth}
\begin{figure}[th]
	\centering
	%\begin{centering}
	\includegraphics[width=0.9\textwidth]{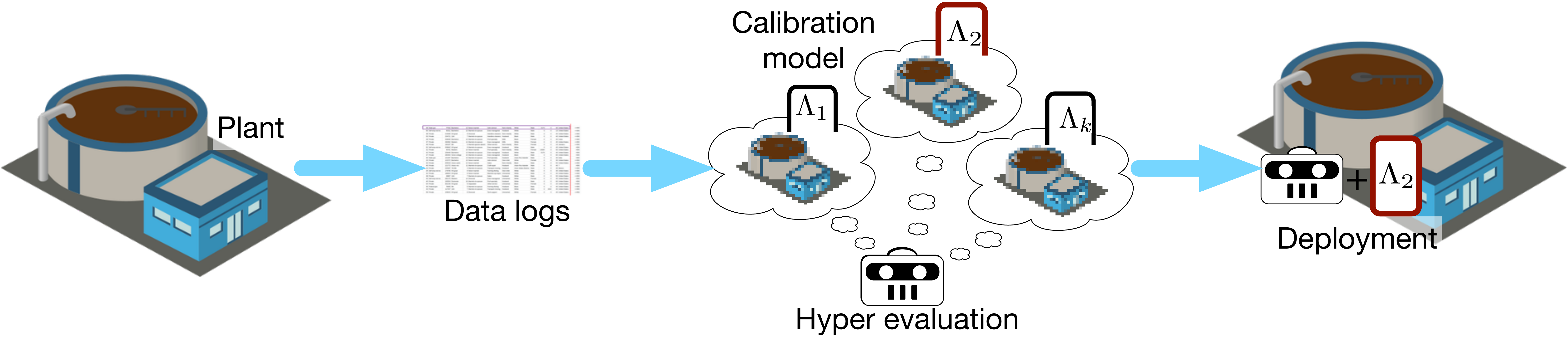}
	\caption{\label{plant_diagram}{\footnotesize
			Data2Online: each hyperparameter setting is denoted by $\Lambda$. The plant imagined by the agent (via calibration model) is intentionally pixilated to reflect approximation of the true plant.}
	}  
	%	\end{centering}
\end{figure}
%\end{wrapfigure}
Consider designing a learning system for controlling a water treatment plant, given only a set of data logs visualized in Figure \ref{plant_diagram}. We want an agent to control pump speeds, mixing tanks, and chemical treatments to clean the water with minimal energy usage---but how do we set the learning rate and other hyperparameters of this agent? We can learn a calibration model offline from data logs previously collected while human operators controlled the plant. The calibration model can be treated like any simulator to develop a learning system, including setting the hyperparameters for learning in deployment. 

Though this is a simple and natural idea, to the best of our knowledge, it is the first general approach for Data2Online. It is common in reinforcement learning to learn models for offline policy evaluation or for planning. These approaches, however, do not need to tackle a key problem we consider in this work: iterating the model for thousands of learning steps. There is only one other work considering how to use offline data to evaluate an online agent \citep{mandel2016offline}, but it is only effective for short horizon problems.

In this paper, we first introduce our Data2Online strategy, and outline conditions on the calibration model and learning agents for this strategy to be effective. We bound the difference in the value in the real environment of the hyperparameters chosen in the calibration model, to the true best hyperparameters, in terms of the calibration model error and length of interaction. We then develop a non-parametric calibration model, based on k-nearest neighbors and a carefully chosen distance metric, that satisfies the conceptual criteria needed for the calibration model. We investigate the approach in three problems with different types of learning dynamics, two different learning agents, under different offline data collection policies, and with ablations on the key components of our proposed calibration model. We conclude by highlighting that grid search can be replaced with any hyperparameter optimization algorithm, and that this further improves performance in the real environment.

\section{Related Problem Settings}\label{sec_setting}
Offline RL involves learning from a dataset, but with a different goal: to deploy a fixed (near-optimal) policy. As a result, the hyperparameter selection approaches for offline RL are quite different. One strategy that has been used is to evaluate different policies corresponding to different hyperparameter settings, under what has been introduced as the Offline Policy Selection problem \citep{yang2020offline}. This setting differs from our setting in that they evaluate the utility of the learned policy, rather than the utility of hyperparameters for learning online, in deployment. Some work in offline RL examines learning from data, and then adapting the policy online \citep{ajay2020opal,yang2021representation,lee2021offlinetoonline}, including work that alternates between data collection and high confidence policy evaluation \citep{chandak2020towards,chandak2020optimizing}. Our problem is complementary to these, as a strategy is needed to select their hyperparameters.

% MARTHAC: This was said above
%deployment scenario is different than our problem of study. In Batch RL the objective is to learn one or more policies offline and then demonstrate that the new policy is an improvement over the data collection policy with no learning in deployment \citep{dudik2011doubly}\footnote{Our hyperparameter tuning problem appears similar to \citet{yang2020offline}'s Offline Policy Selection problem. \citet{yang2020offline} use a offline RL algorithm to learn and evaluate several different policies corresponding to different hyperparameter settings. The key difference is those hyperparameters will never be deployed in an online learning system, as we do in this paper.}. In our offline hyperparameter selection problem we select hyperparameters useful for online learning in deployment. Batch RL algorithms rely on hyperparameters and those can be tuned offline \citep{irpan2019off,yang2020offline}; although it is common to access the deployment environment to sweep said hyperparameters \citep{wu2019behavior}.

In Sim2Real the objective is to construct a high fidelity simulator of the deployment setting, and then transfer the policy and, in some cases, continue to fine tune in deployment. 
We focus on learning the calibration model from collected data, whereas in Sim2Real the main activity is designing and iterating on the simulator itself \citep{peng2018sim}. Again, however, approaches for Data2Online are complementary, and even provide another avenue to benefit from the simulator developed in Sim2Real to pick hyperparameters for fine tuning. 

Domain adaptation in RL involves learning on a set of source tasks, to transfer to a target task. The most common goal has been to enable zero-shot transfer, where the learned policy is fixed and deployed in the target task \citep{pmlr-v70-higgins17a,xing2021domain}. Our problem has some similarity to domain adaptation, in that we can think of the calibration model as the source task and the real environment as the target task. Domain adaptation, however, is importantly different than our Data2Online problem: (a) in our setting we train in a \emph{learned} calibration model not a real environment, and need a mechanism to learn that model (b) the relationship between our source and target is different than the typical relationship in domain adaptation
and (c) our goal is to select and transfer hyperparameters, not learn and transfer policies.

Learning from demonstration (LfD) and imitation learning involve attempting to mimic or extract a policy at least as good as a demonstrator. If the agent is learning to imitate online, then it is unrealistic to assume the demonstrator would generate enough training data required to facilitate hyperparameter sweeps. If the learner's objective is to imitate from a dataset, then this is exactly the problem study of this paper. Unfortunately, hyperparameter tuning in LfD is usually not addressed; instead it is common to use explicit sweeps \citep{merel2017learning, barde2020adversarial,ghasemipour2020divergence,behbahani2019learning} or manual, task-specific tuning \citep{finn2017one}. Hyperparameter tuning, however, is a major obstacle to the deployment of LfD methods \citep{ravichandar2020recent} 

Finally, there is a large literature on hyperparameter selection in RL. Most introduce meta algorithms that learn hyperparameters, including work on meta-descent for stepsizes \citep{sutton1992adapting,xu2018metagradient,jacobsen2019meta} and selecting the trace parameter \citep{downey2010temporal,mann2016adaptive,white2016greedy}. These algorithms could be beneficial for offline hyperparameter selection, because they help reduce sensitivity to hyperparameters; but they are not a complete solution as they still have hyperparameters to tune. Other work has provided parameter-free methods that have theoretically defined formulas for hyperparameters \citep{papini2019smoothing}. Deriving such algorithms is important, but is typically algorithm specific and requires time to extend to broader classes of algorithms, including new advances; it remains useful to consider how to tune hyperparameters for a problem. Finally, recent work has examined online hyperparameter selection, using off-policy learning to assess the utility of different hyperparameters in parallel \citep{paul2019fast,tang2020online}. Otherwise, much of the work has been focused on settings where it is feasible to obtain multiple runs under different hyperparameters---such as in simulation---with the goal to improve on simple grid search \citep{srinivas2010gaussian,bergstra2012random,snoek2012practical,li2018hyperband,jaderberg2017population,falkner2018bohb,parker-holder2020provably}. 

\section{Problem Formulation}\label{sec_formulation}

%\martha{(a) Introduce MDP formalism, for interaction. (b) Explain the offline hyperparameter setting, where we assume we have a batch of data before the agent interacts in the real-world.}

In RL, an {agent} learns to make decisions through interaction with an {environment}. 
We formulate the problem as a Markov Decision Process (MDP), described by the tuple $(\mathcal{S}, \mathcal{A}, \mathcal{R}, \mathcal{P})$. $\mathcal{S}$ is the state space and $\mathcal{A}$ the action space. $\mathcal{R}: \mathcal{S}\times \mathcal{A} \times \mathcal{S} \rightarrow \mathbb{R}$ is the reward, a scalar returned by the environment. 
%$p$ is the transition dynamics function such that $p(s', r | s, a)$ gives the probability of transitioning the agent to state $s'$ with a reward $r$ if it takes action $a$ in state $s$. 
The transition probability $\mathcal{P}: \mathcal{S}\times \mathcal{A} \times \mathcal{S}\rightarrow [0,1]$ describes the probability of transitioning to a state, for a given state and action. On each discrete timestep $t$ the agent selects an action $A_t$ in state $S_t$, the environment transitions to a new state $S_{t+1}$ and emits a scalar reward $R_{t+1}$.
%$\gamma$ is the discounting factor such that $\gamma \in [0, 1]$. 

The agent's objective is to find a policy that maximizes future reward. A policy $\pi: S \times A \rightarrow [0, 1]$ defines how the agent chooses actions in each state. The objective is to maximize future discounted reward or the \emph{return}, 
%$G_t \doteq \sum _{k=0}^{\infty} \prod_{j=1}^k \gamma(S_{t+j}, A_{t+j}, S_{t+j+1}) R_{t+k+1}$)
$G_t \doteq R_{t+1} + \gamma_{t+1} G_{t+1}$
for a discount $\gamma_{t+1} \in [0,1]$ that depends on the transition $(S_t, A_t, S_{t+1})$~\citep{white2017unifying}. For continuing problems, the discount may simply be a constant less than 1. For episodic problems the discount might be 1 during the episode, and become zero when $S_t, A_t$ leads to termination. Common approaches to learn such a policy are Q-learning and Expected Sarsa, which approximate the action-values---the expected return from a given state and action---and Actor-Critic methods that learn a parameterized policy (see \citep{sutton2018reinforcement}).

We additionally assume that in the \emph{Data2Online} setting the agent has access to an offline log of data that it can use to initialize hyperparameters before learning online. This log consists of $\numdata$ tuples of experience $\dataset = \{(S_t, A_t, R_{t+1}, S_{t+1}, \gamma_{t+1})\}_{i=1}^\numdata$, generated by interaction in the environment by a previous controller or controllers. For example, an agent that will use Expected Sarsa might want to use this data to decide on a suitable stepsize $\alpha$, the number of layers $l$ in the neural network (NN) architecture for the action-values and even an initialization $\theta_0$ for the NN parameters---namely a policy initialization. There are several options for each hyperparameter combination, $\lambda = (\alpha, l, \theta_0)$,  resulting in a set of possible hyperparameters $\Lambda$ to consider. This set can be discrete or continuous, depending on the underlying ranges. For example, the agent might want to consider any $\alpha \in [0,1]$ and a $\theta_0$ only from a set of three possible choices. 

Procedurally, the Data2Online algorithm is given the dataset $\dataset$ and the set of hyperparameters $\Lambda$, and outputs a selected hyperparameter setting $\tilde{\lambda}$. A good choice is one that is within $\varepsilon$-optimal of the best hyperparameters
\begin{equation}
%\max_{\lambda \in \Lambda} \text{Perf}(\lambda) - \text{Perf}(\tilde{\lambda}) \le \varepsilon \max_{\lambda \in \Lambda} \text{Perf}(\lambda)
\text{Perf}(\tilde{\lambda}) \ge \max_{\lambda \in \Lambda} \text{Perf}(\lambda) - \varepsilon
\end{equation}
where $\text{Perf}(\lambda)$ is the online performance of the agent, when deployed with the given hyperparameters. Typically, this will be the cumulative reward for continuing problems and average return for episodic problems, for a fixed number of steps $T$. Many hyperparameters may allow the agent to perform well, so we focus on nearly-optimal performance under $\tilde{\lambda}$ rather than on identifying the best hyperparameters. 

The central question for this Data2Online problem is: how can the agent use this log of data to select hyperparameters before learning in deployment? This is no easy task. The agent cannot query $\text{Perf}(\lambda)$. It is not only evaluating a \emph{fixed} policy, for which we could use the large literature on Off-policy Policy Evaluation. It is evaluating a \emph{learning} policy. In the remainder of this paper, we introduce and test a new algorithm for this Data2Online problem. 

%Talk about Markov state property?
%Talk about how in reality, we only get observations?
%Talk about finite MDPs?
%Talk about the offline batch setting.

\section{Data2Online using Calibration Models}\label{sec_calibration}

This section introduces the idea of calibration models and how they can be used for hyperparameter selection. We first discuss how to use the calibration model to select hyperparameters, before providing a way to learn the calibration model. We then discuss certain criteria on the calibration model and agent algorithm that make this strategy more appropriate. We conclude with some theoretical characterization of the error in hyperparameter selection, based on inaccuracies in the calibration model.
%Next, we propose a non-parametric strategy to obtain a calibration model from a batch of offline data. 

\subsection{Using Calibration Models to Select Hyperparameters}\label{sec_calib_select}

A calibration model is a simulator of the environment---learned from an offline batch of data---used to specify (or calibrate) hyperparameters in the agent. With the calibration model, the agent can test different hyperparameter settings. It evaluates the online performance of different hyperparameter choices in the calibration model and selects the best one. It can then be deployed into the true environment without any remaining hyperparameters to tune.

% JAMESB: I chagned the following from H to \Lambda since that's what we use elsewhere.

The basic strategy is simple: we train a calibration model, then do grid search in the calibration model and pick the top hyperparameter, as summarized in Algorithm \ref{alg_selection}. For each hyperparameter, we obtain a measure of the online performance of the agent across $\numruns$ in the calibration model, assuming it gets to learn for $\numsteps$ of interaction. The pseudocode for AgentPerfInEnv is in Algorithm \ref{alg_perf}, for the episodic setting where we report average returns during learning. Note that we add a cutoff in the evaluation scheme to guarantee that at least 30 episodes are seen by the agent during training in the calibration model. We cut off episodes that run longer than $\numsteps/30$ steps, teleporting the agent back to a start state.
%we given an overview on how we can use calibration models to pick hyperparameters. The basicWe first train $\numensembles$ calibration models, using different subsamples of the offline dataset $\dataset$. We obtain performance estimates in multiple calibration models, to obtain a more robust estimate of hyperparameter performance. For each hyperparameter setting and calibration model, we obtain a measure of the online performance of the agent across $\numruns$ in that calibration model, assuming it gets to learn for $\numsteps$ of interaction. The recorded performance for the hyperparameter is the worst-case performance across the $\numensembles$ calibration models, though other statistics such as median performance could also be used. Finally, we return the hyperparameter that obtained the best performance. 

\begin{algorithm}[t]
	\caption{Hyperparameter Selection with Calibration Models using Grid Search}\label{alg_selection}
	\textbf{Input:} $\mathbf{\Lambda}$: hyperparameter set for learner $Agent$\\
	$\dataset$: the offline log data\\
	$\numsteps$: number of interactions or steps \\
	$\numruns$: number of runs 
	\begin{algorithmic}
		\STATE Train calibration model $\cmodel$ with $\dataset$
		\FOR {$\lambda$ in $\mathbf{\Lambda}$}
		\STATE $\text{Perf}[\lambda]$ = AgentPerfInEnv($\cmodel$, $Agent(\lambda)$, $\numsteps$, $\numruns$)
		\ENDFOR			
	\end{algorithmic}
	\textbf{Return:} $\arg\max_{\lambda \in \mathbf{\Lambda}} \text{Perf}[\lambda]$
\end{algorithm}
%\begin{algorithm}[htb]
%\begin{wrapfigure}[26]{L}{0.73\textwidth}
%\vspace{-1cm}
%\begin{minipage}{0.73\textwidth}
\begin{algorithm}[t]
	\caption{AgentPerfInEnv}\label{alg_perf}
	\textbf{Input:} $\cmodel$: calibration model, $Agent$: learner, 
	$\numsteps$: \# of steps, $\numruns$: \# of runs 
	\begin{algorithmic}
		\STATE $\numcutoff \gets \numsteps/30$ $\quad \triangleright$ \text{Ensure there are at least 30 episodes}	
		\STATE $ReturnsAcrossRuns = []$
		\FOR {$i = 1 \ldots \numruns$}
		\STATE $t \gets 0, G \gets 0, Returns = []$
		\STATE $s \gets $ random start state from $\cmodel$
		\FOR {$j = 1 \ldots \numsteps$}
		\STATE Obtain $a$ from $Agent(s)$, obtain $s', r = \cmodel(s,a)$, give $s', r$ to $Agent$
		\STATE $G \gets G + \gamma^{t} r$
		\STATE $t \gets t + 1$
		\IF{$s'$ is terminal or $t > \numcutoff$}
		\STATE Append $G$ to $Returns$
		\STATE $s \gets $ random start state from $\cmodel$, $t \gets 0, G \gets 0$
		\ENDIF
		\ENDFOR
		%\State // Only append last return if above average
		%\State If $G > average(Returns)$, then append $Ret$ to $Returns$
		\STATE $ReturnsAcrossRuns[i] \gets average(Returns)$
		\ENDFOR
	\end{algorithmic}
	\textbf{Return:} $average(ReturnsAcrossRuns)$
\end{algorithm}
%\end{minipage}
%\end{wrapfigure}

Many components in this approach are modular, and can be swapped with other choices. For example, instead of expected return during learning (online performance), optimizing the hyperparameters might be more desirable to find the best policy after a budget of steps. This would make sense if cumulative reward during learning in deployment was not important.
%evaluate offline performance. In other words, the primary goal for the online agent might be to gather data to efficiently learn the optimal policy, off-policy with Q-learning. For this setting, the performance for one run in the calibration model would correspond to the performance of the final (greedy) policy, evaluated by using rollouts in the calibration model. 
We might also want a more robust agent, and instead of expected return, we may want median return. Finally, the grid search can be replaced with a more efficient hyperparameter selection method; we discuss this further in Section \ref{sec_cem}. 

We can also make this hyperparameter search more robust to error in the calibration model by obtaining performance across an ensemble of calibration models. This involves using $\numensembles$ random subsets of the log data, say by dropping at random 10\% of samples, and training $\numensembles$ calibration models. The hyperparameter performance can either be averaged across these models, or a more risk-averse criterion could be used like worst-case performance. Using an average across models is like using a set of source environments to select hyperparameters---rather than a single source---and so could improve transfer to the real environment. 

These are all additions to this new Data2Online strategy. The central idea is to use this calibration model to evaluate hyperparameters, as if we had a simulator. We, therefore, investigate this idea first in its simplest form with expected returns, grid search and only one calibration model.
%We test this addition in our experiments, but start with the simplest approach of using one calibration model for the majority of the work.

\subsection{When is this Approach Effective?}
%\subsection{Criteria for Designing the Calibration Model and Picking Hyperparameter Sets}
\label{limitations}
This section highlights three conceptual criteria for designing the calibration model and selecting agents for which Data2Online should be effective. This includes 1) stability under model iteration, 2) {handling actions with low data coverage} and 3) selecting agent algorithms that only have initialization hyperparameters, namely those that affect early learning but diminish in importance over time. 
%We highlight throughout that the criteria for a calibration model are not the same as for the model in model-based RL, nor for Sim2Real. 
%We highlight throughout that the calibration model need not be highly accurate---though using the true environment would be the ideal scenario---and that 

Producing reasonable transitions under many steps of model iteration is key for the calibration model. The calibration model is iterated for many steps, because the agent interacts with the calibration model as if it were the true environment---for an entire learning trajectory. It is key, therefore, that the calibration model be \emph{stable} and \emph{self-correcting}. A stable model is one where, starting from any initial state in a region, the model remains in that region. A self-correcting model is one that, even if it produces a few non-real states, it comes back to the space of real states. Otherwise, model iteration can produce increasingly non-sensical states, as has been repeatedly shown in model-based RL \citep{talvitie2017self,jafferjee2020hallucinating,abbas2020selective,chelu2020forethought}.  

The model also needs to handle actions with no coverage, or low coverage. For unvisited or unknown states, the model simply does not include such states. The actions, however, can be queried from each state. If an action has not been taken in a state, nor a similar state, the model cannot produce a reasonable transition. Any choice will have limitations, because inherently we are filling in this data gap with an inductive bias. A common choice in offline RL is to assume pessimistic transitions: if an action is not observed, it is assumed to have a bad outcome. This ensures the learned, fixed policy avoids these unknown areas. 

The choice is even more nuanced in Data2Online. Just like offline RL, it can be sensible to avoid these unknown actions, to answer: in the space known to the agent, what hyperparameters allow the agent to learn quickly? But, another plausible alternative is that we want to select hyperparameters to encourage the agent to explore unknown areas, since once deployed the agent can update its policy in these unknown areas. In other words, a principle of optimism could also be sensible. Selecting the right inductive bias will depend on the environment and the types of hyperparameters we are selecting. This question will likely be one of the largest questions in Data2Online, similarly to how it remains an open question in offline RL. 

The third criterion is a condition on the agent, rather than the model.
Practically, we can only test each hyperparameter setting for a limited number of steps in the calibration model. So, the calibration model is only simulating early learning. 
This suggests that this approach will be most effective if we tune \emph{initialization hyperparameters}: those that provide an initial value for a constant but wash away over time. Examples include an initial learning rate which is then adapted; policy initialization; and an initial architecture that is grown and pruned over time. 

These criteria are conceptual, based on reasoning through when we expect success or failure. We use these conceptual criteria to propose an appropriate approach to learn a calibration model in the next section. In addition to conceptual reasoning, theoretical understanding of the Data2Online problem setting is also critical. We provide a first step in that direction in the next section. 

\subsection{Theoretical Insights}

This problem has aspects that both make it harder and potentially easier than a standard offline RL problem. One aspect that makes this problem harder is that we have to evaluate a learning policy offline, rather than a fixed learned one. A fixed policy can be assessed using policy evaluation, and there exists a variety of theoretical results  on the quality of those estimates, many based on the foundational simulation lemma \citep{kearns2002nearoptimal}. No such results exist for evaluating a policy that will be changing online. 

At the same time, intuitively, the problem could be easier than the offline RL problem, because the policy adapts online. Instead, we only have to select from a potentially small number of hyperparameters, rather than from a potentially large policy space. For example, it may be relatively easy to identify the best stepsize out of a set of three stepsizes. Further, if a policy learning algorithm is chosen that is robust to its hyperparameter settings, then the problem may be even simpler. For example, it may be simple to select the initial stepsize for an adaptive stepsize algorithm, where the initial stepsize only influences the magnitude of updates for the first few steps. 

We first extend the foundational simulation lemma to the Data2Online setting, in Theorem 1. Then, in Theorem 2, we show how to use this result, to bound how far the value of the hyperparameters chosen in the learned calibration model are from the best hyperparameters. Finally, we discuss how it might be possible to formalize this second intuition, for future theoretical investigation. 

\newcommand{\norm}[1]{\| #1 \|}
\newcommand{\cS}{\mathcal{S}}
\newcommand{\cA}{\mathcal{A}}
\newcommand{\cH}{\mathcal{H}}
\newcommand{\RR}{{\mathbb{R}}}
\newcommand{\EE}{{\mathbb{E}}}
\newcommand{\rmax}{{r_{\text{max}}}}

We start by defining some needed terms.
An online learner can be characterized by a history dependent policy (see Chapter 38 of \citep{lattimore2020bandit}). A history dependent policy is $\pi=(\pi_0,\pi_1,\pi_2,\dots)$ where $\pi_t:\cH_t\to\Delta(\cA)$ and $\cH_t=(\cS\times\cA\times\RR)^{t}\times\cS$ is the history at time step $t$. For simplicity, we assume the rewards are deterministic in $[0,\rmax]$ and the MDP has one initial state $s_0$. The online learning agent interacts with the environment for $T$ steps in total, in either a continuing or fixed-horizon setting.
%We use the $T$ step objective to select hyperparameters. 

The value function for this online learner $\pi$ is the sum of rewards from time $t$ to the end of learning at $T-1$
\begin{align*}
V^\pi_{t}(h_t) &= \EE\left[\sum_{t'=t}^{T-1} r(S_{t'},A_{t'}) \mid H_t=h_t\right] 
\end{align*}
where the expectation is under $A_{t}\sim\pi_{t}(\cdot\mid H_t)$, $S_{t+1}\sim P(\cdot\mid S_t,A_t)$.
%  if $t$ is not the last step of an episode and $S_{t+1}=s_0$ is $t$ is the last step of an episode.
Note that $V^\pi_{0}(s_0)$ is the $T$ step objective that we use to select hyperparameters. For the fixed-horizon setting where episodes are of length \textit{horizon}, $T=K\cdot\textit{horizon}$ where $K$ is the number of episodes. In this setting, the expectation is under $S_{t+1}\sim P(\cdot\mid S_t,A_t)$ if $t$ is not the last step of an episode and $S_{t+1}=s_0$ if $t$ is the last step of an episode. Dividing $V^\pi_0(s_0)$ by $K$ gives the average episodic return over $K$ episodes. 

\begin{theorem}[Simulation Lemma for Online Learners] Assume the rewards $r(s,a)$ are deterministic in $[0,\rmax]$ and the MDP has one initial state $s_0$
	Suppose we have a learned model $(\hat P,\hat r)$ such that 
	\begin{equation*}
	\norm{\hat P(\cdot\mid s,a) - P(\cdot\mid s,a)}_1 \leq \varepsilon_p \quad \text{and} \quad |r(s,a)-\hat r(s,a)|\leq\varepsilon_r \quad\quad \text{for all $(s,a)\in\cS\times\cA$}
	\end{equation*} 
	and that $\hat r(s,a) \in [0, \rmax]$.
	Let $\hat V^\pi_{t}(\cdot)$ denote the value function under the learned model. Then for any history dependent policy $\pi$, we have that
	\begin{equation*}
	\lvert V^\pi_0(s_0) - \hat V^\pi_0(s_0)\rvert \leq \varepsilon_r T + \frac{\varepsilon_p \rmax T^2}{2}.
	\end{equation*}
	% $\lvert V^\pi_0(s_0) - \hat V^\pi_0(s_0)\rvert \leq \varepsilon T V^\pi_{0}(s_0)$.
\end{theorem}
\begin{proof}
	We follow the proof of the simulation lemma. Since the rewards are deterministic, the history does not need to contain reward, that is, $\cH_t=(\cS\times\cA)^t\times\cS$.
	For any $h_t = (s_0,a_t,\dots,s_t)\in\cH_t$ with $t<T-1$,
	% if $t$ is not the last step of an episode
	\begin{align*}
		V^\pi_t(h_t) &= \sum_{a\in\cA} \pi_t(a \mid h_t) r(s_t,a) + \sum_{a\in\cA} \pi_t(a \mid h_t) \sum_{s'\in\cS} P(s,a,s') V^\pi_{t+1}((h_t,a,s')) \\
		\hat V^\pi_t(h_t) &= \sum_{a\in\cA} \pi_t(a \mid h_t) \hat r(s_t,a) + \sum_{a\in\cA} \pi_t(a \mid h_t) \sum_{s'\in\cS} \hat P(s,a,s') \hat V^\pi_{t+1}((h_t,a,s'))
	\end{align*}
	and for the last step we have 
	\begin{equation*}
		V^\pi_{T-1}(h_{T-1}) = \sum_{a\in\cA} \pi_{T-1}(a \mid h_{T-1}) r(s_{T-1}, a) \quad \text{and}\quad
		\hat V^\pi_{T-1}(h_{T-1}) = \sum_{a\in\cA} \pi_{T-1}(a \mid h_{T-1}) \hat r(s_{T-1},a).
	\end{equation*}
	% and, if $t$ is the last step of an episode
	% \begin{align*}
	% 	V^\pi_t(h_t) &= \sum_{a\in\cA} \pi_t(a \mid h_t) r(s_t,a) + \sum_{a\in\cA} \pi_t(a \mid h_t) V^\pi_{t+1}((h_t,a,s_0)) \\
	% 	\hat V^\pi_t(h_t) &= \sum_{a\in\cA} \pi_t(a \mid h_t) r(s_t,a) + \sum_{a\in\cA} \pi_t(a \mid h_t) \hat V^\pi_{t+1}((h_t,a,s_0)).
	% \end{align*}
	%
	We prove the simulation lemma from the last step. For $t=T-1$,
	\begin{align*}
		\left\lvert V^\pi_{T-1}(h_{T-1}) - \hat V^\pi_{T-1}(h_{T-1}) \right\rvert 
		= \left\lvert \sum_{a\in\cA} \pi_{T-1}(a \mid h_{T-1}) r(s_{T-1},a) -  \sum_{a\in\cA} \pi_{T-1}(a \mid h_{T-1}) \hat r(s_{T-1},a) \right\rvert \leq \varepsilon_r.
	\end{align*}
	%
	% For all $t<T$ and $t$ is the last step of an episode,  
	% \begin{align*}
	% 	\lvert V^\pi_t(h_t) - \hat V^\pi_t(h_t) \rvert
	% 	=& \lvert \sum_{a\in\cA} \pi_t(a \mid h_t) r(s_t,a) + \sum_{a\in\cA} \pi_t(a \mid h_t) V^\pi_{t+1}((h_t,a,s_0))\\
	% 	& -\sum_{a\in\cA} \pi_t(a \mid h_t) r(s_t,a) - \sum_{a\in\cA} \pi_t(a \mid h_t) \hat V^\pi_{t+1}((h_t,a,s_0))\rvert \\
	% 	\leq& \max_a \lvert V^\pi_{t+1}((h_t,a,s_0)) - \hat V^\pi_{t+1}((h_t,a,s_0)) \rvert
	% \end{align*}
	%
	For all $t<T-1$,
	% and $t$ is not the last step of an episode, 
	\begin{align*}
		&\lvert V^\pi_t(h_t) - \hat V^\pi_t(h_t) \rvert  \\
		&= \lvert \sum_{a\in\cA} \pi_t(a \mid h_t) r(s_t,a) + \sum_{a\in\cA} \pi_t(a \mid h_t) \sum_{s'\in\cS} P(s,a,s') V^\pi_{t+1}((h_t,a,s')) \\
		&\quad-\sum_{a\in\cA} \pi_t(a \mid h_t) \hat r(s_t,a) - \sum_{a\in\cA} \pi_t(a \mid h_t) \sum_{s'\in\cS} \hat P(s,a,s') \hat V^\pi_{t+1}((h_t,a,s')) \rvert \\
		&\leq \varepsilon_r + \sum_{a\in\cA} \pi_t(a \mid h_t) \lvert \sum_{s'\in\cS} P(s,a,s') V^\pi_{t+1}((h_t,a,s')) - \sum_{s'\in\cS} \hat P(s,a,s') \hat V^\pi_{t+1}((h_t,a,s')) \rvert \\
		&= \varepsilon_r + \sum_{a\in\cA} \pi_t(a \mid h_t)  \lvert \sum_{s'\in\cS} P(s,a,s') V^\pi_{t+1}((h_t,a,s')) - \sum_{s'\in\cS} \hat P(s,a,s') V^\pi_{t+1}((h_t,a,s')) \\
		&\quad + \sum_{s'\in\cS} \hat P(s,a,s') V^\pi_{t+1}((h_t,a,s')) - \sum_{s'\in\cS} \hat P(s,a,s') \hat V^\pi_{t+1}((h_t,a,s')) \rvert \\
		&\leq \varepsilon_r + \sum_{a\in\cA} \pi_t(a \mid h_t) \lvert \sum_{s'\in\cS} (P(s,a,s') - \hat P(s,a,s')) \underbrace{V^\pi_{t+1}((h_t,a,s'))}_{\leq (T -t - 1)\rmax}\rvert \\
		% &\leq \sum_{a\in\cA} \pi_t(a \mid h_t) \lvert \sum_{s'\in\cS} (P(s,a,s') - \hat P(s,a,s')) \underbrace{V^\pi_{t+1}((h_t,a,s'))}_{\leq V^\pi_{0}(s_0)}\rvert \\
		&\quad + \sum_{a\in\cA} \pi_t(a \mid h_t) \lvert \sum_{s'\in\cS} \hat P(s,a,s') (V^\pi_{t+1}((h_t,a,s')) - \hat V^\pi_{t+1}((h_t,a,s')))\rvert \\
		&\leq \varepsilon_r + \sum_{a\in\cA} \pi_t(a \mid h_t) (T - t - 1) \rmax \sum_{s'\in\cS} \mid P(s,a,s') - \hat P(s,a,s')\mid + \max_{a,s'} \lvert V^\pi_{t+1}((h_t,a,s')) - \hat V^\pi_{t+1}((h_t,a,s'))\rvert \\
		&\leq \varepsilon_r + \varepsilon_p \rmax(T -t - 1) + \max_{a,s'} \vert V^\pi_{t+1}((h_t,a,s')) - \hat V^\pi_{t+1}((h_t,a,s')) \rvert
		% &\leq \varepsilon V^\pi_{0}(s_0) + \max_{a,s'} \vert V^\pi_{t+1}((h_t,a,s')) - \hat V^\pi_{t+1}((h_t,a,s')) \rvert.
	\end{align*}
	Therefore, $\lvert V^\pi_{0}(s_0) - \hat V^\pi_{0}(s_0) \rvert \leq \varepsilon_r + \varepsilon_p \rmax (T - 1) + \dots + \varepsilon_r + \varepsilon_p \rmax 1 + \varepsilon_r \leq \varepsilon_r T + \frac{\varepsilon_p \rmax T^2}{2}$.
	% Therefore, $\lvert V^\pi_{0}(s_0) - \hat V^\pi_{0}(s_0) \rvert \leq \varepsilon V^\pi_{0}(s_0) + \dots + \varepsilon V^\pi_{0}(s_0) \leq \varepsilon T V^\pi_{0}(s_0)$.
\end{proof}

Theorem 1 says that if we have a good model of the environment, we can evaluate the $T$ step objective for any online learner with bounded error. In particular, we can control this evaluation error by controlling the error of the learned model. Note that $v_{\text{max}} \doteq T \rmax$ is the maximum value, so the last term $\frac{\varepsilon_p \rmax T^2}{2}$ can be interpreted as $\tfrac{\varepsilon_p v_{\text{max}} T}{2}$, meaning the bound scales with $T$: $(\varepsilon_r + \tfrac{\varepsilon_p v_{\text{max}}}{2}) T$.

% \begin{corollary}
% 	Given a policy $\pi$, suppose we use $n$ runs to estimate $\hat V^\pi_0(s_0)$, and denote the estimator by $\tilde V^\pi_0(s_0)$,  then $\lvert \hat V^\pi_0(s_0) - \tilde V^\pi_0(s_0)\rvert \leq T\sqrt{\frac{\ln{(2/\delta)}}{2n}}$ with probability by $1-\delta$. 
% \end{corollary}
% \begin{proof}
% 	It follows directly from the Hoeffding's inequality.
% \end{proof}

Back to our problem setting. Let $\Lambda$ be the set of hyperparameters and $\pi_{\lambda}$ be a learner's policy with $\lambda\in\Lambda$. In our algorithm, we choose the best hyperparameters based on $\tilde V^\pi_0(s_0)$, which is an estimator for $\hat V^\pi_0(s_0)$ by running $n$ runs with $\hat P$. Let $\tilde \lambda = \arg\max_{\lambda\in\Lambda} \tilde V^{\pi_\lambda}_0(s_0)$ be the hyperparameters returned by our algorithm and $\lambda^* = \arg\max_{\lambda\in\Lambda} V^{\pi_{\lambda}}_0(s_0)$ be the true best hyperparameters in the set. The following theorem shows that our hyperparameters will not be too far from the best hyperparameters in terms of the $T$ step objective. 
\begin{theorem}
	Under the same conditions as Theorem 1, with probability $1-\delta$, we have
	\begin{align*}
		V^{\pi_{\lambda^*}}_0(s_0) - V^{\pi_{\tilde \lambda}}_0(s_0) &\leq 2\varepsilon_r T + \varepsilon_p \rmax T^2 + \rmax T \sqrt{\frac{2\ln{(4/\delta)}}{n}}\\
		&= \underbrace{(2\varepsilon_r + \varepsilon_p v_{\text{max}}) T}_{\text{approximation error}} + \underbrace{v_{\text{max}} \sqrt{\frac{2\ln{(4/\delta)}}{n}}}_{\text{estimation error}}.  
	\end{align*}
	% \begin{align*}
	% 	V^{\pi_{\lambda^*}}_0(s_0) - V^{\pi_{\tilde \lambda}}_0(s_0) \leq 2 \varepsilon T V^{\pi_{\lambda^*}}_0(s_0) + T \sqrt{\frac{2\ln{(4/\delta)}}{n}}. 
	% \end{align*}
\end{theorem}

\begin{proof}	
	By Hoeffding's inequality, for a given $\pi$, we have with probability $1-\delta/2$ that
\begin{equation*}
\lvert \hat V^\pi_0(s_0) - \tilde V^\pi_0(s_0)\rvert \leq T\rmax\sqrt{\frac{\ln{(4/\delta)}}{2n}}
\end{equation*}	
because the return in each run, to give the sample average $\tilde V^\pi_0(s_0)$, is in $[0,T \rmax]$. 
Using the union bound, we can say this inequality holds for both $\pi_{\lambda^*}$ and $\pi_{\tilde \lambda}$, with probability $1-\delta$. The source of this difference is from using a limited number of runs to approximate $\hat V^\pi_0(s_0)$. As we increase the number of runs $n$, then the difference between our estimator $\tilde V^\pi_0(s_0)$ and $\hat V^\pi_0(s_0)$ goes to zero.

%	and the union bound, for $\pi=\pi_{\lambda^*},\pi_{\tilde \lambda}$, we have $\lvert \hat V^\pi_0(s_0) - \tilde V^\pi_0(s_0)\rvert \leq T\sqrt{\frac{\ln{(4/\delta)}}{2n}}$ with probability $1-\delta$. 
	Now we can reason about the hyperparameters chosen using $\tilde \lambda = \arg\max_{\lambda\in\Lambda} \tilde V^{\pi_\lambda}_0(s_0)$.
	\begin{align*}
		V^{\pi_{\lambda^*}}_0(s_0) - V^{\pi_{\tilde \lambda}}_0(s_0)
		&= V^{\pi_{\lambda^*}}_0(s_0) - \tilde V^{\pi_{\tilde \lambda}}_0(s_0) + \tilde V^{\pi_{\tilde \lambda}}_0(s_0) - V^{\pi_{\tilde \lambda}}_0(s_0) \\
		&\leq V^{\pi_{\lambda^*}}_0(s_0) - \tilde V^{\pi_{\lambda^*}}_0(s_0) + \tilde V^{\pi_{\tilde \lambda}}_0(s_0) - V^{\pi_{\tilde \lambda}}_0(s_0) \\
		&= V^{\pi_{\lambda^*}}_0(s_0) - \tilde V^{\pi_{\lambda^*}}_0(s_0) + \hat V^{\pi_{\lambda^*}}_0(s_0) - \hat V^{\pi_{\lambda^*}}_0(s_0) \\
		& \ \ \ \  + \tilde V^{\pi_{\tilde \lambda}}_0(s_0) - V^{\pi_{\tilde \lambda}}_0(s_0) + \hat V^{\pi_{\tilde \lambda}}_0(s_0) - \hat V^{\pi_{\tilde \lambda}}_0(s_0) \\
		&\leq \lvert V^{\pi_{\lambda^*}}_0(s_0) - \hat V^{\pi_{\lambda^*}}_0(s_0)\rvert + \lvert\hat V^{\pi_{\tilde \lambda}}_0(s_0) - V^{\pi_{\tilde \lambda}}_0(s_0)\rvert \\
		& \ \ \ \ + \lvert \hat V^{\pi_{\lambda^*}}_0(s_0) - \tilde V^{\pi_{\lambda^*}}_0(s_0)\rvert + \lvert \tilde V^{\pi_{\tilde \lambda}}_0(s_0) - \hat V^{\pi_{\tilde \lambda}}_0(s_0) \rvert \\
		&\leq 2\max_{\lambda\in\Lambda} \lvert V^{\pi_{\lambda}}_0(s_0) - \hat V^{\pi_{\lambda}}_0(s_0)\rvert + 2T\rmax\sqrt{\frac{\ln{(4/\delta)}}{2n}} \\
		&\leq 2\varepsilon_r T + \varepsilon_p \rmax T^2 + \rmax T\sqrt{\frac{2\ln{(4/\delta)}}{n}}.
		% &= T( 2\varepsilon_r + \varepsilon_p T + \sqrt{\frac{2\ln{(4/\delta)}}{n}} )
		% &\leq 2\max_{\lambda\in\Lambda} \varepsilon T V^{\pi_{\lambda}}_0(s_0) + 2T\sqrt{\frac{\ln{(4/\delta)}}{2n}} \\
		% &\leq 2 \varepsilon T V^{\pi_{\lambda^*}}_0(s_0) + T\sqrt{\frac{2\ln{(4/\delta)}}{n}}.
	\end{align*}	
	The last inequality follows from Theorem 1. 
\end{proof}

This result is a sanity check that we can reason about error in identifying hyperparameters based on model error. However, it has several limitations. One limitation is that the result is for continuing problems and fixed-horizon episodic problems, but not variable length episodic problems. The analysis does not address variable length episodic problems, because it would impact both the histories on which policies are conditioned as well as the definition of the value function for this learning policy. The more important limitation, however, is that the bound depends on the length of interaction $T$, and worse on the squared length of interaction $T^2$ if we assume $v_{\text{max}} \doteq T \rmax$. Even if episodes are short, we expect the agent to learn for thousands of steps and so $T$ can be quite large. It may be difficult to obtain sufficiently low $\varepsilon$ (model error) to control for this accumulating error over many steps.

Ideally, we should be able to obtain a better result by considering smoothness in performance with respect to hyperparameters. Empirical studies do suggest performance changes smoothly as a function of hyperparameters, whenever hyperparameter sensitivity plots are shown. We hypothesize that there exists a subset of hyperparameters such that $V^{\pi_{\lambda}}_0(s_0)$  is smooth w.r.t. the hyperparameters in the subset, and hyperparameters outside the subset have very low $V^{\pi_{\lambda}}_0(s_0)$. Therefore, the error bound from Theorem 2 just needs to be smaller than the performance gap between hyperparameters in the good subset and hyperparameters outside the good subset to guarantee finding a nearly optimal hyperparameter setting. This direction is an important next step.

\section{Stable Calibration Models with KNNs}\label{sec_knn}

We develop a non-parametric k-nearest neighbor (KNN) calibration model that (a) ensures the agent always produces real states and (b) remains in the space of states observed in the data, and so is stable under many iterations. The idea is simple: the entire offline data log constitutes the model, with trajectories obtained by chaining together transitions. Figure \ref{fig:calibration_ep} shows the interaction between the calibration model and the agent in each episode.
There are, however, several nuances in using this strategy to obtain calibration models. In particular, the method relies heavily on the distance metric used to find nearest neighbors. Further, the dataset is limited, and may have poor coverage for actions in certain states. We start by introducing the basic approach, and then discuss these two nuances in the following two subsections.  

\begin{figure*}[t]
	\begin{centering}
		\vspace{-1cm}
		\includegraphics[width=0.8\textwidth]{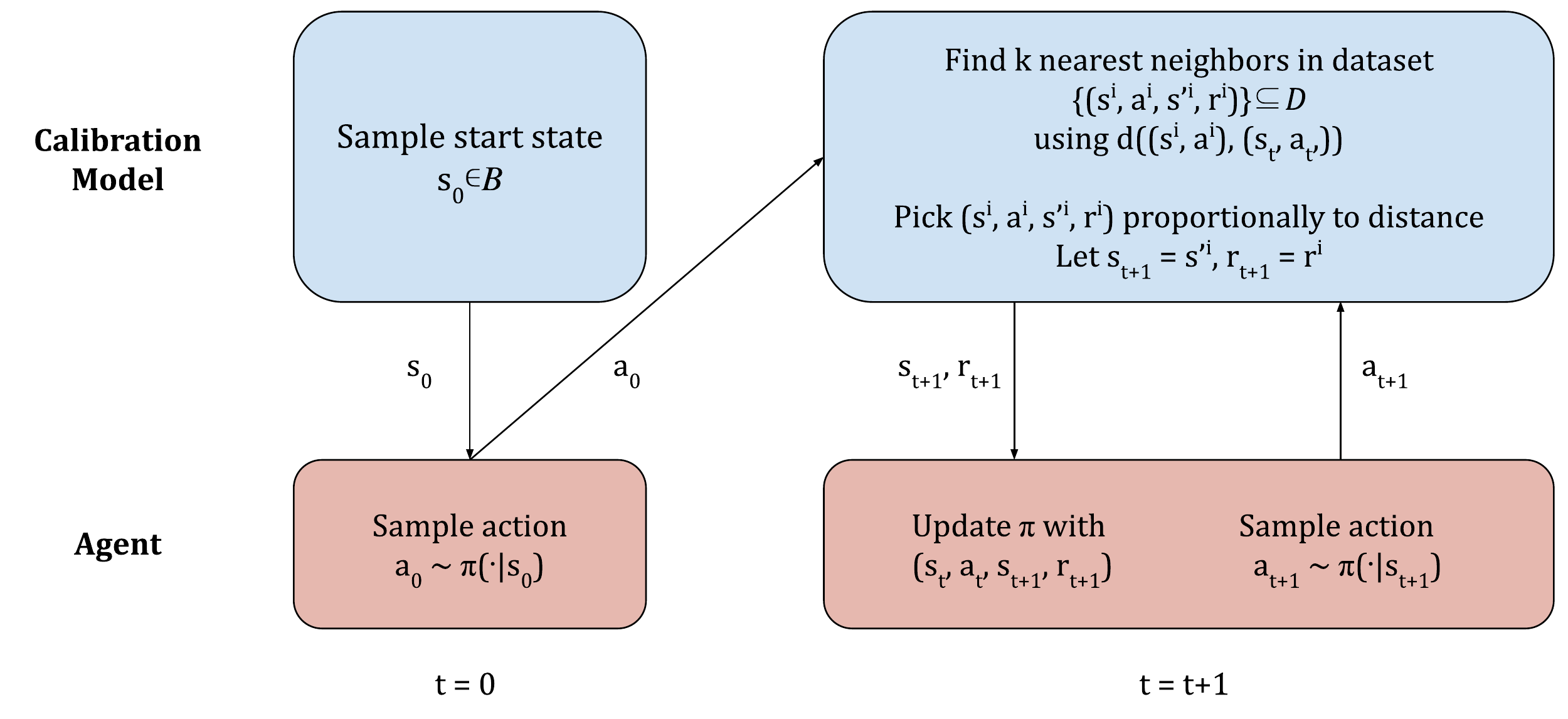}
		\caption{The interaction between the calibration model and the agent in each episode. The process happening in the calibration model is on the first row, the agent's process is on the second row.
		}
		\label{fig:calibration_ep}
	\end{centering}
\end{figure*}

%A calibration model is supposed to (a) ensure the agent always produces real states and (b) remain in the space of states observed in the data, and so is stable under many iterations. At every timestep, the agent is brought to a real transition in the dataset, to avoid creating nonexistent states in the state distribution. The agent is prevented from going to areas uncovered by the dataset by assuming these areas are dangerous. 

\subsection{The KNN Calibration Model}

%\begin{algorithm}[t]
%	\caption{KNNCalibration}\label{alg_knn_calib}
%	\textbf{Input:} $Agent$: the learner, 
%	$\numsteps$: number of steps, $\mathcal{D}$: dataset 
%	\begin{algorithmic}
%		\State Assume the dataset is a list of tuples of the form $(S, A, S', R, T )$
%		\State $\psi \leftarrow LaplaceRepTraining(S, A, S', T)$
%		\State $\Phi \leftarrow \psi(S)$
%		\State $Trees \leftarrow KDTreeConstruction(\Phi, A, S', R, T)$
%		\State Extract starting states $\mathcal{B} \subseteq \mathcal{D}$
%		\State $Reset \leftarrow True$
%		\For{each step in [$\numsteps$]}
%		\If {$Reset$}
%		\State Sample $s \in \mathcal{B}$
%		\State $Reset \leftarrow False$
%		\EndIf
%		\State $a \leftarrow Agent(s)$
%		\State $s', r, termination \leftarrow \textit{Algorithm \ref{simulatorstep}}(a, Trees, k)$
%		\If {$termination$}
%		\State $Reset \leftarrow True$
%		\EndIf
%		\State $s \leftarrow s'$
%		\EndFor
%	\end{algorithmic}
%\end{algorithm}

The calibration model needs to produce a (stochastic) outcome next state and reward $r, s'$, given a state and action $s,a$. We can produce novel trajectories from a dataset, by noting that if a state-action pair $(s_t,a_t)$ is similar to $(s_i,a_i)$ for a stored tuple $(s_i, a_i, r_i, s'_i)$, then it is plausible that $r_i, s'_i$ could also have been observed from $(s_t,a_t)$. To allow for stochastic transitions, the $k$ most similar pairs to $(s_t,a_t)$ can be found, and the next state and reward selected amongst these by sampling proportionally to similarity. 

More formally, given the current state-action pair $(s_t, a_t)$, the model searches through all tuples $(s, a, s', r)$ and selects the $k$ nearest neighbors, according to similarity between $(s_t,a_t)$ and $(s,a)$. (We will discuss how to compute similarity in Section \ref{sec_similarity}.) Let $\{(s_i, a_i, r_i, s'_i)\}_{i=1}^k$ correspond to these tuples, and $d_i$ to the distance between $(s_t, a_t)$ and $(s_i, a_i)$. 
Then these $(r_i, s'_i)$ are all possible outcome rewards and next states, where the likelihood corresponds to similarity to $(s_i, a_i)$. If $(s_t,a_t)$ is very similar to $(s_i, a_i)$, then $(r_i, s'_i)$ is a likely outcome. Otherwise, the more dissimilar, the more unlikely it is that $(r_i, s'_i)$ is a plausible outcome. The tuple $i$ is sampled proportionally to the probability $p_i = 1- \frac{d_i}{\Sigma_{j \in [1,k]} d_j}$, where a smaller distance indicates higher similarity. 

This procedure is summarized in Figure \ref{fig:calibration_ep}. At the start of each episode, a start state $s_0$ is sampled randomly from the set of start states in the dataset. The agent takes its first action $a_0$, to get the first pair $(s_0, a_0)$ and the $k$ nearest neighbors are found. This process continues until the agent reaches a terminal state, or the episode is cutoff and the agent teleported back to a start state.
An overview of learning the KNN calibration model is given in Algorithm \ref{alg_knn_calib} and sampling the model in Algorithm \ref{alg_knn_sample}. 

There are several details worth mentioning in the algorithms. First, a KNN model relies heavily on an appropriate distance. For example, for input states that correspond to $(x,y)$ position, Euclidean distance can be a poor measure of similarity. If there is a wall in the environment, two states might be similar in Euclidean distance, but far in terms of reachability with notably different dynamics. We ameliorate this by learning a new representation---called a Laplace representation---$\psi(s)$ and using Euclidean distance in this new space that better reflects similarity in transition dynamics, as described in Section \ref{sec_similarity}. 

Second, there may be no similar pairs in the data for a given $(s,a)$. The state is one that is observed in the dataset, but the action may not be since it is selected by the agent running in the calibration model. When the next outcome state $s_{t+1}$ is chosen from $s_t$, the agent selects action $\tilde{a}_{t+1}$. The dataset might contain multiple transitions from states like $s_{t+1}$---including of course the transition that includes $s_{t+1}$---but these may be for only a subset of the actions. If none of these transitions uses $\tilde{a}_{t+1}$, then the dataset has insufficient coverage to infer what might occur when taking that action in the environment. When this occurs in Algorithm \ref{alg_knn_sample}---when the closest point (minimum distance) is too far away (above a threshold)---we set the return to a default return and terminate the episode. We discuss an appropriate choice for this default return in Section \ref{sec_datacoverage}.   

Finally, we want to ensure that the model is efficient to query, even if we have a large dataset. For the discrete action setting, it is possible to get an $O(1)$ look-up by caching the nearest neighbors upfront. For $n$ datapoints, for each action we construct a table with $n$ rows and $k$ columns for the nearest neighbors, where each neighbor is stored as its index from 1 to $n$. Each transition consists of jumping between rows, using these indices. The detailed pseudocode is provided in Appendix \ref{apdx:calibration-detail}. More generally, for continuous actions, we can use a k-d tree \citep{bentley1975multidimensional} to search for the k-nearest neighors. The k-d tree takes the transformed state-action pair, $(\psi(s),a)$ as the key for the search. For a dataset of size $n$, it costs $O(n\log n)$ to construct the k-d tree and $O(\log n)$ to query for a nearest neighbor. This low computational complexity is key to allow us to use all of the data to create our calibration model. 
%To further improve the running time in the discrete action space, we can construct a nearest-neighbors dictionary when initializing the calibration model, containing all nearest-neighbors and distances for each transition in the dataset. Therefore, the calibration model constantly searches in the dictionary during learning, and the searching time is reduced to $O(1)$. The pseudocode for this change is provided in Appendix \ref{apdx:calibration-detail}.
%It is worthwhile using a k-d tree to search for nearest points if $n$ is much bigger than the input dimension.

\begin{algorithm}[t]
	\caption{Learn KNN Calibration Model}\label{alg_knn_calib}
	\textbf{Input:} dataset $\dataset$ with tuples of the form $(S, A, S', R)$, number of nearest neighbors $k$ (default $k = 3$) \\
	\textbf{Constructs:} Representation $\psi$ for distances, KD-Tree $Trees$ for fast nearest neighbors search, $R_{\text{default}}$ default return
	\begin{algorithmic}
		\STATE $\psi \leftarrow $LaplaceRepTraining$(\dataset)$ in Algorithm \ref{reptraining}
		\STATE $Trees \leftarrow $KDTreeConstruction$(\psi, \dataset)$ in Algorithm \ref{treeconstr}
		\STATE Extract starting states $\mathcal{B} \subseteq \mathcal{D}$
		\STATE Set $R_{\text{default}}$ to minimum return in dataset (pessimistic default return)
	\end{algorithmic}
\end{algorithm}

\begin{algorithm}[t]
	\caption{Sample KNN Calibration Model}\label{alg_knn_sample}
	\textbf{Input:} State $s_t$, Action $a$; if no action is given, procedure returns a start state
	\begin{algorithmic}
		\IF {No action is given}
			\RETURN Sample $s \in \mathcal{B}$ uniform randomly
		\ENDIF
		\STATE // Find $k$ nearest neighbors to $s_t,a_t$, to get potential next states and rewards
		\STATE ${ (s'_i, r_i, d_i)}_{i=1}^k \leftarrow$ KDTreeSearch$(\psi(s_t), a, Trees, k)$ in Algorithm \ref{treesearch} 
		\STATE // If closest neighbor is far, then return a default return and terminate
		\IF{$\min_i d >$ threshold}
		\STATE $r \leftarrow R_{\text{default}}$,  $s' \leftarrow \text{terminal}$
		\RETURN $(r, s')$
		\ENDIF
		\STATE Sample $ i \in [1,k]$ according to probabilities $p_i = 1- \frac{d_i}{\Sigma_{j \in [1,k]} d_j}$
		\RETURN $(r_i, s'_i)$
	\end{algorithmic}
\end{algorithm}

We can contrast this KNN calibration model to two natural alternatives: a kernel density estimator (KDE) model and a neural network model. A KDE model is a non-parametric estimator, that has even been investigated for model-based RL \citep{pan2017rem}. Like our KNN calibration model, it should also stably remain within the region defined by the dataset. However, unlike the KNN calibration model, a KDE calibration model could produce non-existent states. It effectively interpolates between observed datapoints, and so results in significant generalization. If we consider again the example with $(x,y)$ position in a gridworld with walls, then the KDE calibration model could produce transitions within the wall. 

Another alternative is to use a neural network (NN) to learn a calibration model. The dataset can be used to learn the expected next state and reward, for given a state and action, using regression on inputs $(s,a)$ and targets $(r, s')$. Or, to obtain a distribution over next states and rewards, a conditional distribution can be learned using mixture density networks or stochastic networks. Such NN models can produce non-existent states, just like the KDE model. Worse, however, is that iteration under the NN model may not be stable. Several works in model-based RL have illustrated that iterating such models can produce less and less plausible outcomes states \citep{talvitie2017self,jafferjee2020hallucinating,abbas2020selective,chelu2020forethought}. 

Avoiding such issues is an active area of research, such as by training models to be correct over multiple steps of iteration \citep{talvitie2014model,venkatraman2015improving,talvitie2017self,williams2017information,ke2019learning}. For our Data2Online setting, this is difficult to use, because the number of steps of iteration is much larger than what is typically used in model-based RL. A model in model-based RL may be rolled out for 100 steps, whereas here it needs to be rolled out for the entire length of training. Though there is some work investigating learning neural networks that are guaranteed to be stable under iteration \citep{manek2019learning}, the approach requires a specialized architecture. We are hopeful that, with more research, NN models will become a viable choice for learning calibration models. We include NNs in our experiments, as a baseline. 

%The calibration model is a framework, which provides a possibility to do hyperparameter selection without interaction with the true environment. The components of calibration model types can vary, such as the transition generation component. It is technically possible to replace KNN with a method that complies with the criteria in Section \ref{limitations}. Note that the neural network is excluded from the options, considering that the calibration model should return the actual state. The neural network is a common choice in model learning and a powerful approximator in reinforcement learning literature, but it does not meet the calibration model's criteria. The prediction of the neural network model can be an unreal state, and there is limited chance for it to move back to the actual state space as long as one prediction is out-of-distribution. A long rollout increases the risk of error accumulation in the prediction. In addition, there is no guarantee that the agent can learn in the true state space with this calibration model.

%\subsection{Ensuring the Accuracy of Nearest Neighbors with Reliable Distance Metrics}\label{sec_similarity}
\subsection{Improving the Distance Metric for the KNN}\label{sec_similarity}

It is not hard to imagine examples where Euclidean distance on states or observations does not appropriately reflect similarity of states. For example, in a maze environment, if inputs correspond to $(x,y)$, two nearby points in Euclidean distance may actually be on opposite sides of a wall, thus far apart in terms of transition dynamics. Similarly, Euclidean distance does not apply to images, since pixel-wise difference can make every image look very different from all the others in the dataset. 

Instead, we exploit a standard approach in metric learning: we first map the inputs to a new space where Euclidean ($\ell_2$) distance is meaningful. In particular, we would like a new representation $\psi(s)$ where states $s_i$ and $s_j$ that have similar outcomes in terms of states and rewards are mapped to similar vectors $\psi(s_i) \approx \psi(s_j)$, and ones with dissimilar outcomes are mapped to different representations. 

Such representations that reflect similarity in terms of transition dynamics have been explored under what are called \emph{Laplace representations} \citep{wu2018laplacian}. The approach relies on having a stored trajectory that maintains the order of interaction. The objective includes two components: an \emph{attractive term} that encourages two states that are nearby in the trajectory to have similar representations, and a \emph{repulsive term} that encourages randomly sampled states to have different representations. For a neural network $\psi_\theta$ with parameters $\theta$, the last layer of the NN $\psi_\theta(s)$ has loss
%\begin{align*}
%&\sum_{s_i \sim \dataset, u \sim P} \| \psi_\theta(s_i) - \psi_\theta(s_{i+u}) \|_2^2 \\ 
%&+ \sum_{s_i, s_j \sim \dataset} \left(\beta(\psi_\theta(s_i)^T \psi_\theta(s_{j}))^2 - \zeta\|\psi_\theta(s_i)\|_2^2 - \zeta \|\psi_\theta(s_j)\|_2^2 \right)
%\end{align*}
{\small
	\begin{align*}
	\sum_{s_t \sim \dataset} \| \psi_\theta(s_t) - \psi_\theta(s_{t+1}) \|_2^2  
	+ \sum_{s_i, s_j \sim \dataset} \left((\psi_\theta(s_i)^T \psi_\theta(s_{j}))^2 - \|\psi_\theta(s_i)\|_2^2 - \|\psi_\theta(s_j)\|_2^2 \right)
	\end{align*}}
%where $\beta$ and $\zeta$ controls weights of each term, and $P$ assigns higher probability for states closer to $s_i$ in the trajectory, the simplest case is to always sample $u=1$.
The inclusion of representation norms $-\|\psi_\theta(s_i)\|_2^2$ ensures that the representation is not simply decreased to zero to satisfy the first attractive term. Minimizing this objective encourages $\| \psi_\theta(s_t) - \psi_\theta(s_{t+1}) \|_2^2$ to be small for states right beside each other in the trajectory---temporally close. The second term $(\psi_\theta(s_i)^T \psi_\theta(s_{j}))^2$ is the repulsive term that encourage random pairs to have orthogonal representations. It is possible for $s_t, s_{t+1}$ to be randomly selected for the second term, but this is not that likely under the possible $n^2$ pairs; the first term dominates, ensuring these nearby points have similar representations.
% and the second term primarily encourages temporally distant states to have dissimilar representations. 
More details on learning the Laplace representation are given in Appendix \ref{app_laplace}.

The distance for a state-action pair is defined differently for discrete and continuous actions. 
For discrete actions, two actions are considered similar only when they are exactly the same. The resulting distance is
\begin{equation*}
d((s_i,a_i), (s_j, a_j)) = \begin{cases} d_s(s_i, s_j) & \text{if } a_i = a_j \\ \infty  & \text{else}\end{cases} \quad \text{ for } d_s(s_i, s_j) \doteq \| \psi(s_i) - \psi(s_j)\|_2^2
\end{equation*}
%We maintain a separate k-d tree for each action. 
In practice, we simply keep separate data structures for each action, to find nearest neighbors.
For continuous action problems, the Laplace representation can actually be learned on $(s,a)$ directly, to obtain $\psi(s,a)$.

\subsection{Insufficient Data Coverage}\label{sec_datacoverage}
We do not require the dataset to have perfect state and action space coverage. 
% MARTHAC: Maybe this is all just sort of obvious
%Therefore, not all areas in the state and action spaces are covered by the dataset. For a KNN calibration model, all knowledge about the true environment comes from the offline dataset, including the state distribution, reward range, and the termination conditions. That means if there is no data coverage in the dataset, the calibration model cannot accurately predict the consequence of entering this area. 
We only query the KNN calibration model from states $s$ that are in the dataset, by construction. But, for a given action $a$, there may be no state-action pair that is similar to $(s,a)$ and so there is insufficient information about the outcome for that pair. What then should the model return?
%For example, the model could be too optimistic when the failure cases are missing in the dataset. Due to the nature of the dataset, it is impossible for the calibration model to learn about the termination case. In the worst case, the calibration model suggests all hyperparameters perform equally well. As a consequence, the online agent may pick a hyperparameter that has a poor performance in the true environment because the hyperparameter might lead the agent to the bad case in the true environment, which is missing in the dataset.

A natural choice is to truncate the episode, provide a default return---as if the agent had managed to visit future states in the environment---and transition back to the start state. This synthetic interaction in the calibration model is inaccurate, so we encourage the agent to learn within the parts of the calibration model that meaningfully reflect the true environment and avoid these unknown areas. This suggests using a \emph{pessimistic} default return. The default return can be set to the minimal return experienced in the dataset. When the agent reaches these unknown state-action pairs, it receives a low return and on the next episode is less likely to come back to this unknown state-action pair. 

Pessimism has also been used in offline RL, but for a subtly different purpose than here. The goal of pessimism in offline RL is to avoid unknown areas, as it is unlikely for the fixed policy to be good in a region that it has never observed and further that unknown region may be dangerous. It is much safer to stay within the data distribution, and simply improve performance within that region. 

For us, the policy can adapt online if it reaches unknown areas, so it is not necessary to ensure the agent avoids them in the environment. But, we avoid encouraging the agent to visit these unknown areas in the calibration model because they are not reflective of the true environment, potentially skewing hyperparameter selection. For example, if the agent was instead encouraged to visit these state-action pairs (using optimism), then it might find an unknown state-action pair and spend most of its time visiting it. The hyperparameters would be tuned to work well for this interaction, with short episodes and (default) Monte-carlo return from this state-action that do not require any bootstrapping. Our primary purpose with this choice, therefore, is to make interaction in the calibration model more similar to interaction in the environment, under the unavoidable limitations of insufficient data coverage.

\section{Experiments}
\label{exp_sec}
We conducted a battery of experiments to provide a rounded assessment of when an approach can or cannot be expected to reliably select good hyperparameters for online learning.
We investigate varying the data collection policy and size of the data logs to mimic a variety of deployment scenarios ranging from a near-optimal operator to random data. We explore selecting hyperparameters of different types for several different agents, and investigate a non-stationary setting where the environment changes from data collection to deployment. 
%We also investigate the policy initialization as a hyperparameter choice.
We begin with the simplest first question: how does our approach compare to simple baselines and with different choices of calibration model type. 

To extensively test the reliability of our approach, we deploy the algorithm on variants of Acrobot, Puddle World, and Cartpole \citep{sutton2018reinforcement}. All three environments are episodic and have a low-dimensional continuous state and discrete actions. Small environments allow extensive experimentation; critical for hyperparameter analysis and achieving statistically significant results. In addition, recent studies have shown that conclusions from small classic control environments match those generated in larger scale environments like Atari \citep{obando2020revisiting}. Experiments were conducted on a cluster and a powerful workstation using $\sim8327$ CPU hours and no GPUs. Full lists of all the hyperparameters can be found in the appendix. 

\subsection{Experiment 1: Comparing Calibration Models}
In this experiment we investigate the benefits of our approach with different choices of model in two classic control environments. We compare our KNN calibration model with learned Laplace similarity metric to an NN model trained to predict the next state and reward given input state and action observed in the calibration data. In addition, we also test an NN calibration model that takes the {\em Laplacian encoding} (see Section \ref{sec_knn}) of the current state as input and predicts the next state and reward to provide the network with a better transition-aware input representation. We used two continuous state, discrete action, episodic deployment environments, Acrobot and Puddle World, as described in the appendix and in introductory texts \citep{sutton2018reinforcement}.

In this first experiment we select the hyperparameters for a linear softmax-policy Expected Sarsa agent (from here on, linear Sarsa) from data generated by a simple policy with good coverage. The agent uses tile coding to map the continuous state variables to binary feature vectors (see Sutton \& Barto \citep{sutton2018reinforcement} for a detailed discussion of tile coding). This on-policy, Sarsa agent learns quickly but is sensitive to several important hyperparameters. We investigate several dimensions of hyperparameters including the step-size and momentum parameters of the Adam optimizer, the temperature parameter of the policy, and the value function weight initialization. We choose these hyperparameters because their impact on performance is somewhat transient and can be overcome by continued learning; this reflects our desire for the agent to continually learn and adapt in deployment. 

We used a near-optimal policy for each environment to collect data for building the calibration models. The near-optimal data collection policy for Acrobot can solve the task in 100 steps, and the near-optimal data collection policy in Puddle World achieves an average return of -25. In both cases the policy will provide the system with many examples of successful trajectories to the goal states in the 5000 transition data log.

%This will never need to be explained again for any experiment that is setup similarily
Our evaluation pipeline involves several steps. First we evaluate the {\em true performance} (steps per episode for Acrobot and return per episode in Puddle World) of each hyperparameter combination in the deployment environment: running for 15,000 steps in Acrobot and 30,000 steps in Puddle World, averaging over 30 runs. 
We then use the data collection policy to generate the calibration data log and learn each model. 
%(The details of the model learning step can be found in the appendix.) We then evaluate each hyperparameter setting using our approach. 
We record the {\em true performance} of the selected hyperparameters to summarize the performance. This whole process---running the data collection policy to generate a data log, learning the calibration model, and evaluating the hyperparameters---is repeated 30 times (giving 30 datasets with 30 corresponding hyperparameter selections). The statistic of interest is the median and distribution of the {\em true performance} for the hyper-parameters selected across runs. In the ideal case, if there is one truly best set of hyperparameters, the system will choose those every time and the variance in {\em true performance} would be zero. 

We also included two baselines. The first is randomly selecting hyperparameters, called Random, to get a sense of the spread in performance; all methods should outperform Random.
We also include an Offline RL algorithm, called Fitted-Q Iteration (FQI) \citep{ernst2005treebased}, that learns a policy from the calibration data and then deploys the learned policy fixed in the deployment environment. For the FQI baseline we simply plot the distribution of performance of each of the 30 extracted policies on the deployment environment. For each policy, we average the performance over 30 random seeds. We tested FQI with a tile coded representation and a NN representation; the tile coded version performed best and we report that result.

\begin{figure*}[t]
	\begin{centering}
		\vspace{-1.3cm}
		\includegraphics[width=0.8\textwidth]{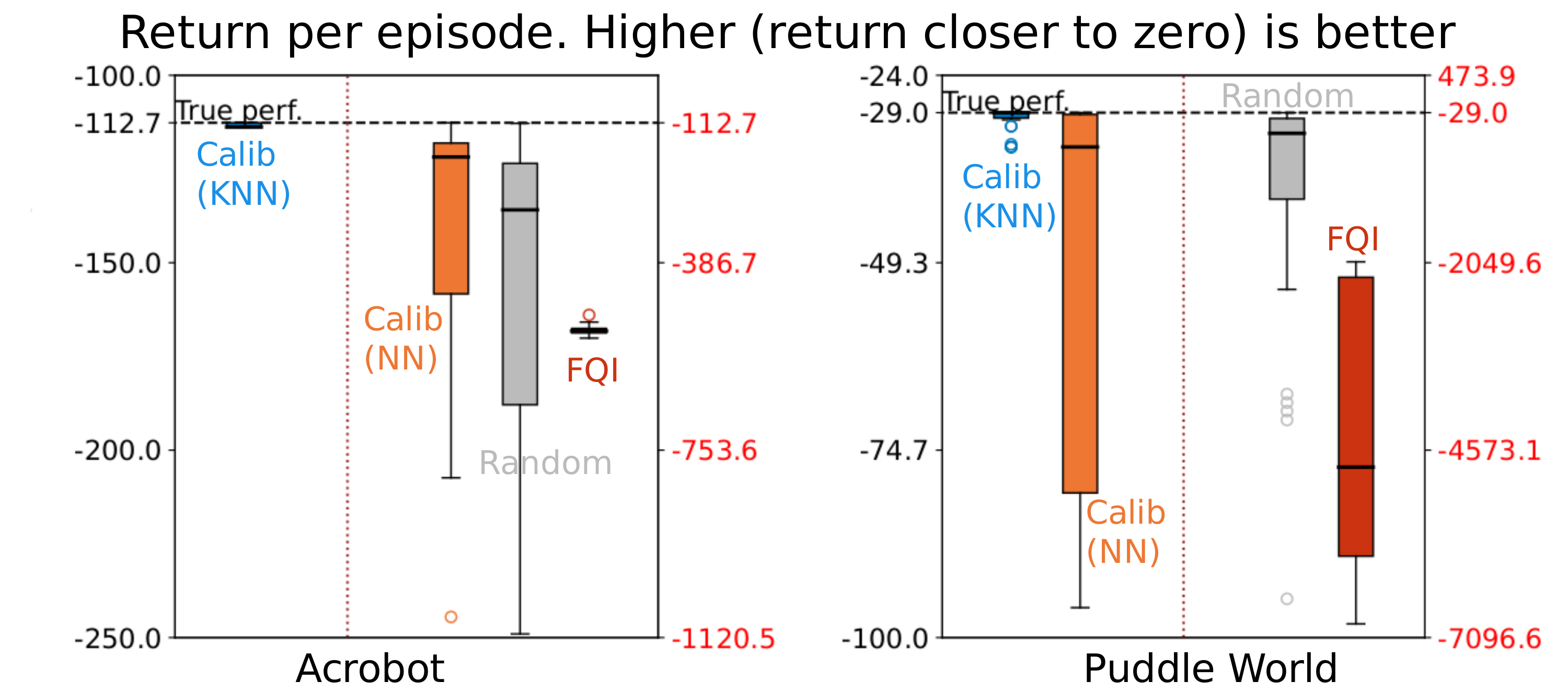}
		
		\caption{\textbf{Hyperparameter transfer with different calibration models.} 
			Each subplot shows the performance of two calibration models compared against two baselines: FQI and random (described in text).
			The dotted horizontal line indicates the average performance of the best hyperparameter setting in the sweep in the deployment environment.
			Each box shows the distribution summarizing the true performance in deployment of the best hyperparameters selected in each run of the experiment.
			%			In Acrobot (lhs) {\bf lower is better}, and in Puddle World (rhs) {\bf higher is better}.
			We plot the return per episode, thus {\bf higher the better}.
			The LHS of each subplot uses the LHS y-axis and the RHS (separated by the dotted vertical orange line) uses the RHS \textcolor{red}{red} y-axis.
			In each subplot the bold line represents the median, the boxes represent the 25th and 75th percentiles, and outliers are represented as circles. 
			% MARTHAC: Too strong a statement
			%If the centre of mass is close to the dotted horizontal line, then the system is choosing hyperparameters well. 
			Low variance indicates that the system reliably chooses the same  hyperparameters each run. The performance of the random baseline characterizes the maximum possible variation. Recall the performance of each hyper is precomputed off-line for each hyperparameter combination and is thus not a source of variation in our setup.
			%In Acrobot (left-hand-side), the performance is measured by the averaged number of steps per episode. A lower number of steps suggests a better performance, while in Puddle World (right-hand-side) we measure the return per episode, thus the higher the better.
		}
		\label{fig:model_works}
	\end{centering}
\end{figure*}

Figure \ref{fig:model_works} summarizes the results. 
In both environments the KNN calibration model performed well, selecting the same hyperparameters as would a sweep directly in the deployment environment. The NN calibration models perform poorly overall. The NN calibration model using raw inputs (no Laplacian encoding) was not as effective, and so we only include results for the NN with the Laplacian encoding in Figure \ref{fig:model_works} and relegate the other to Appendix \ref{apdx:calibration_raw}. Their performance can be unstable, choosing hyperparameters with good performance in some runs, but often choosing poor hyperparameters. 
FQI generally performs worse than even Random. Note that we spent quite a bit of time improving FQI, as well as optimizing over several stepsize and regularization hyperparameters.
%In Acrobot the neural network calibration model using Laplacian encoding as inputs outperforms random hyper selection, whereas the FQI baseline performs much worse. 
This suggests the calibration data log is too limited to extract a good policy and deploy it without additional learning, but the data appears useful for selecting hyperparameters with the KNN calibration model. 

We also used our approach to tune both step-size parameters of a linear Actor-critic agent with tile coding. The KNN calibration model was able to select top performing hyperparameters for Actor-critic in both Acrobot and Puddle World---though the agent performed worse than linear Sarsa (results in Appendix \ref{apdx:actor_critic}).

\subsection{ Experiment 2: Varying Data Collection Policies}
The objective of this experiment was to evaluate the robustness of our approach to changing both the amount of offline data available and the quality of the policy used to collect the data. We experimented with three different policies corresponding to {\em near-optimal}, {\em medium},  and {\em naive} performance to collect 5000 transitions for training our KNN Laplacian calibration model. The near-optimal policy was identical to the one used in the previous experiment. The medium policy was designed to achieve roughly half the visits to goal states after 5000 steps (approximately 90 for Puddle World \& 25 for Acrobot) compared to the near optimal policy. The naive policy was designed such that it achieved significantly fewer visits (approximately 35 for Puddle World \& 12 for Acrobot). We also tried different data log sizes of 500, 1000, and 5000 samples using the medium policy, all shown in Figure \ref{fig:different_datasets_ac}.

\begin{figure*}
	\centering
	\vspace{-1cm}
	\includegraphics[width=0.99\textwidth]{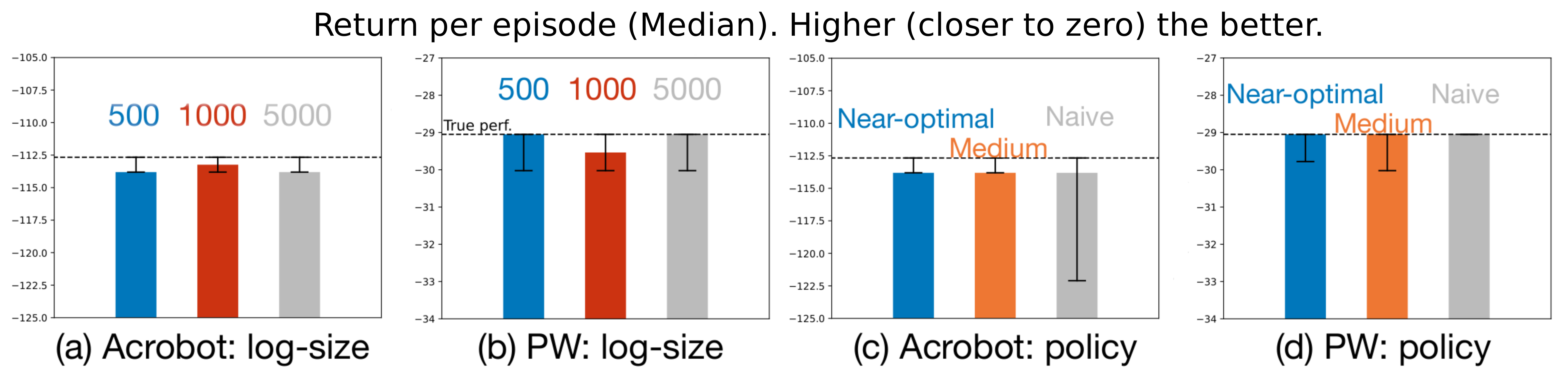}
	\caption{\textbf{The role of data logs.} 
		Plots (a) and (b) show the median deployment performance of hyperparmaters selected from calibration models constructed from different sized data logs (with 25\% and 75\% quartiles). Plots (c) and (d) show the median deployment performance under different policies used to collect data logs for constructing the calibration model. In both subplots, {\bf higher is better}. In these bar plots the median is shown by the top of the colored bar, and the quartiles are shown by the black whiskers. Overall, the utility of the calibration model appears largely insensitive to the data log size and policy in these environemnts; hyperparameters that perform well in deployment can be selected.
		%This plot can be interpreted exactly as Figure \ref{fig:model_works}: the median close to the horizontal dotted line represents good performance and small spread indicates that the system reliably chooses hyperparameters that perform similarity across runs. 
		%Both experiments are done in acrobot. The results for puddle world are shown in Figure \ref{fig:different_datasets_pw}. The y-axis has the same meaning as in Figure \ref{fig:model_works}. 
	}
	\label{fig:different_datasets_ac}
\end{figure*}

The results in Figure \ref{fig:different_datasets_ac} show that our approach is largely insensitive to data log size and policy in these classic environments. Even 500 transitions contains enough coverage of the state-space and successful terminations to produce a useful calibration model. This is in stark contrast to the FQI results in Experiment 1, where a policy trained offline from the same size data log failed to solve either task. 
%We observed similar failures for FQI under other p
%\footnote{Although experiment 1 used the optimal policy, we also tested FQI with the medium policy and observed similar failures.} 
Exploration in both these environments is not challenging; therefore, the success of the calibration model is not surprising. This positive outcome, however, reflects that it may be simpler to pick hyperparameters \emph{in some environments}. In Experiment 4, we investigate a failure case in Cartpole.

\subsection{Experiment 3: When the Environment Changes}
Learning online is critical when we expect the environment to change. This can happen due to wear and tear on physical hardware, un-modelled seasonal changes, or the environment may appear non-stationary to the agent because the agent's state representation does not model all aspects of the true state of the MDP. In this latter case it is often best to track rather than converge; to never stop learning (see \citep{sutton2007role}). In our problem setting, the deployment environment could change significantly between (a) calibration data collection and (b) the deployment phase. Intuitively we would expect batch approaches that simply deploy a fixed policy learned from data to do poorly in such settings. The following experiment simulates such a scenario with the Acrobot environment.

The experiment involves two variants of the environment. As before, we collected 5000 transitions using the near-optimal policy in Acrobot and then applied our approach to select good hyperparameters for the linear Sarsa agent. Unlike before, we evaluate the hyperparameters selected on a second, {\em changed} Acrobot environment. In the changed Acrobot environment we doubled the length and mass of the first link length. Our two phase setup changes the dynamics of Acrobot but does not prevent learning reasonably good policies as you will see in the results below. This whole process was repeated 30 times (generating 30 datasets with a corresponding 30 calibration models) to aggregate the results presented in Figure \ref{fig:nonstationary_acrobot}.

We included three baselines to help contextualize the results: (1) transferring the policy from the first environment, (2) transferring the policy learned in the calibration model, and (3) FQI.
The first baseline, called {\em Sarsa (True)}, simply transfers the policy learned in the first Acrobot environment to the changed Acrobot environment (no calibration model was used, hence the label {\em true}). The second baseline, called {\em Sarsa (Calibration}) simply uses the best performing policy learned by Sarsa in our calibration model, where the calibration model is created with data from the first Acrobot environment. Finally, we also included a FQI baseline. We trained a policy using FQI with tile coding on the data generated from the first environment (the same data used to build the calibration model). Then we evaluated the policy learned by FQI on the changed Acrobot environment. These baselines are meant to illustrate how performance might be effected if the environment dynamics changed but a prelearned policy was applied without taking the changes into account, perhaps because no one noticed the environment had changed. In all three baselines the policy evaluated in the second environment is fixed (no learning in deployment). 

\begin{figure*}[h!]
	\vspace{-0.3cm}
	\begin{minipage}[c]{0.4\textwidth}
		%\vspace{-.3cm}
		\includegraphics[width=0.9\linewidth]{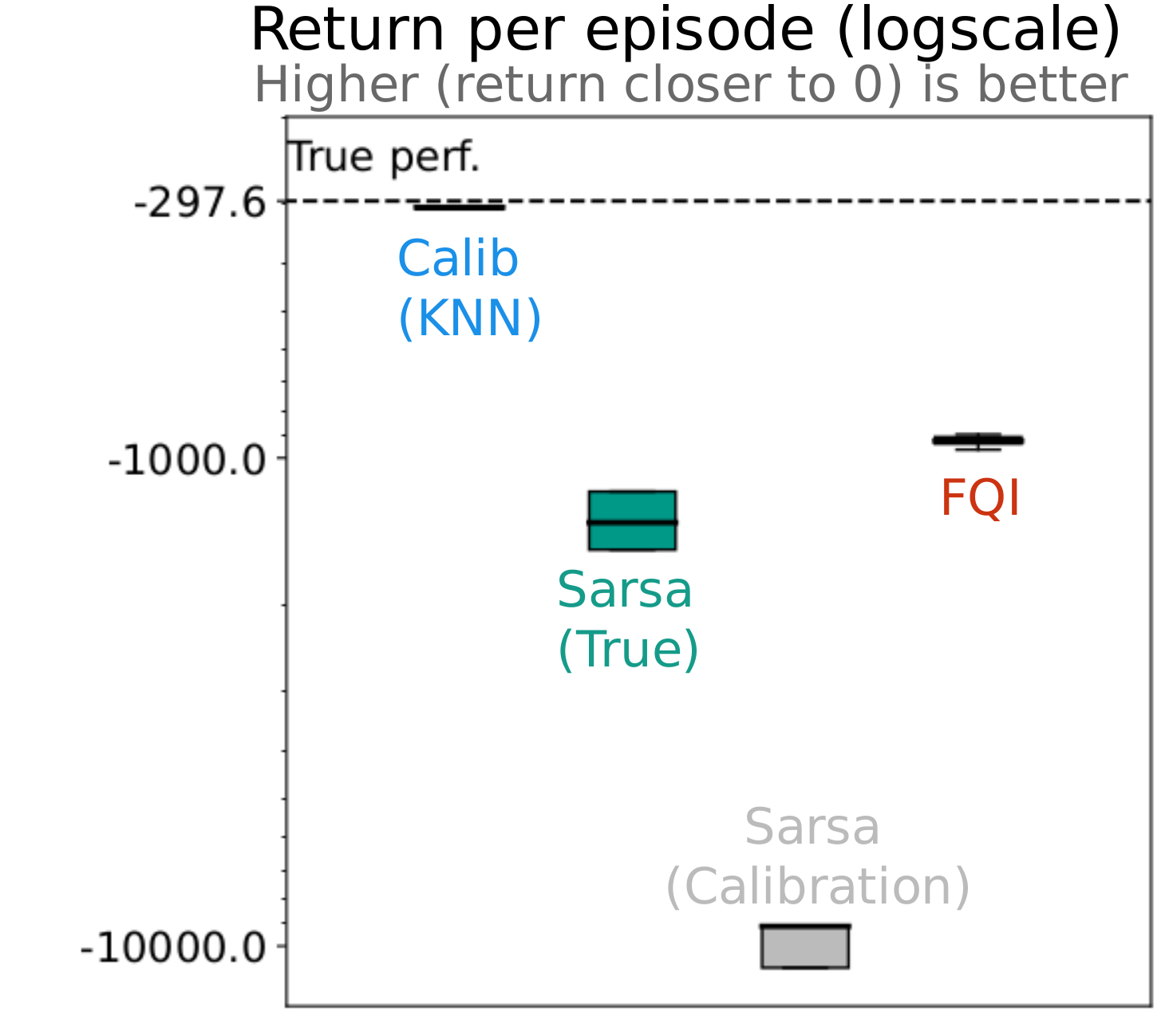}
	\end{minipage}\hfill
	\begin{minipage}[c]{0.6\textwidth}
		
		\caption{
			%			\footnotesize
			\textbf{Selecting hyperparameters in the face of non-stationarity.} The plot above summarizes the performance of our approach compared with three fixed-policy transfer approaches (described in text). 
			%This plot can be interpreted exactly as Figure \ref{fig:model_works}: the median close to the vertical dotted line represents good performance and small spread indicates that the system reliably chooses hypers that perform similarity across runs. 
			Selecting hyperparameters for deployment works well even when the environment changes between calibration data collection and deployment. Deploying fixed policies, on the other hand, performs poorly by comparison.}
		\label{fig:nonstationary_acrobot}
		%\end{figure*}
	\end{minipage}\hfill
	
\end{figure*}

The results in Figure \ref{fig:nonstationary_acrobot} highlight that transferring fixed policies can be problematic when the environment changes. Our calibration-based approach performs best and appears robust under the abrupt non-stationarity tested here. Clearly, the difference between the two environments is significant enough such that transferring a policy directly learned on the first environment (the Sarsa-True baseline) performs worse than using our approach to select hyperparameters in the calibration model and then learning a new policy from scratch. Interestingly, learning and transferring a policy from the calibration model was worse than using than training on the first environment or training from the calibration data (as in FQI). It is not surprising that transferring hyperparameters and learning in deployment is more robust than transferring fixed policies, in these non-stationary settings.

\subsection{Experiment 4: A Failure Case}
\begin{figure*}[t]
	\centering
	\vspace{-1.3cm}
	\includegraphics[width=0.6\textwidth]{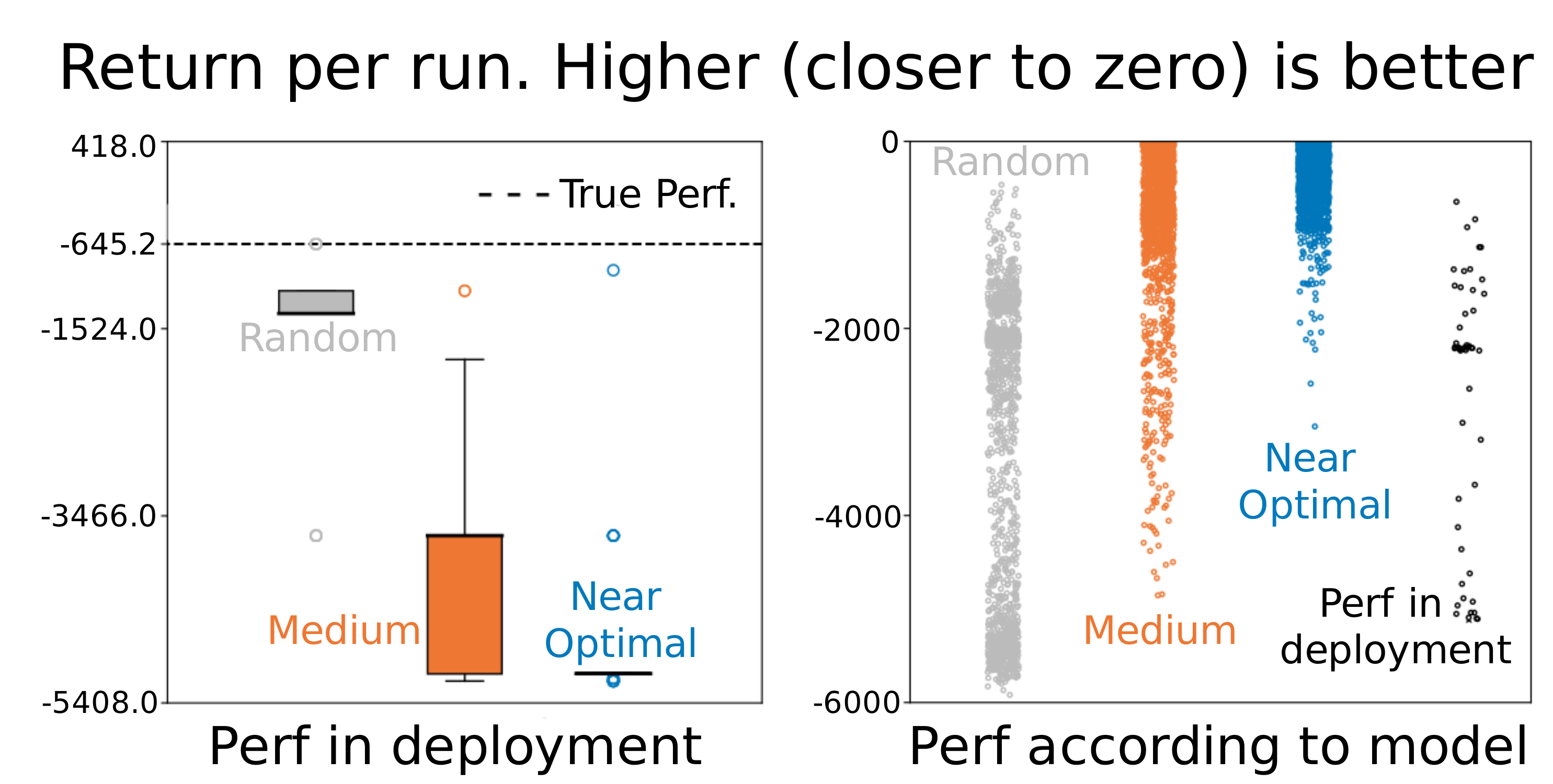}
	\caption{\textbf{Success and failure in Cartpole.} This plot shows performance of three different calibration models constructed from random, near-optimal and medium policies. {\bf Left}: the performance of the hyperparameters in deployment as picked by different calibration models. {\bf Right}: each model's evaluation of all hyperparamters across all 30 runs. Ideally the distribution of performance would match that of the hyperparameter performance in the deployment environment---black dots far right. }
	%Each point represents one hyperparameter setting for each of 30 independent runs.}
	\label{fig:exp_cartpole}
\end{figure*}

Our approach is not robust to all environments and data collection schemes. In this section we investigate when it can fail. One obvious way our approach can fail is if the agent's performance in the calibration model is always the same: no matter what hyperparameter we try, the system thinks they all perform similarly. To illustrate this phenomenon we use the Cartpole environment. In Cartpole, the agent must balance a pole in an unstable equilibrium as long as it can. If the cart reaches the end of the track or the pole reaches a critical angle, failure is inevitable. Near-optimal policies can balance the pole for hundreds of steps rarely experiencing failures and, thus, visit only a small fraction of the state-action space.  A data log collected from the near-optimal policy would likely produce a calibration model where failures are impossible and all hyperparameters appear excellent. 

We test this hypothesis by looking at three data collection policies. We used a random policy, a near-optimal policy with random initial pole angles, a medium policy that was half as good as the near-optimal policy (twice as many failures). We expect the random policy to provide useful data for the calibration model, whereas the near-optimal policy will cause failure due to the above reason. The interim policy helps understand sensitivity to this issue: we might expect the issue to be resolved with a less optimal policy. 

Figure \ref{fig:exp_cartpole} indeed shows that the dynamics of Cartpole combined with particular data collection policies can render the calibration model ineffective
We see in the left-hand plot that the hyperparameter chosen with the calibration model using random data performs somewhat reasonably, but fails for both the medium and near-optimal policies. Even with random starting states the calibration model for near-optimal policy: the calibration model never simulated dropping the pole. The random policy produced the best calibration model. Unsurprisingly, the random policy drops the pole every few steps and thus the log contained many failures and higher state coverage. Nevertheless, the performance was still poor because there were no examples of balancing the pole for many consecutive steps: the model constructed from random data was still a poor model of the true deployment environment. 

We can see this further, by looking at the performance estimates under the three calibration models. The right-hand plot in Figure \ref{fig:exp_cartpole} shows the performance of all the hyperparameters according to the calibration model (before deployment). The blue dots show that most hyperparameters appear good in calibration when the model is constructed with data from a near-optimal policy. At the other extreme the grey dots show a large spread in performance of hyperparameters when the model is constructed with data from a the random policy. Note that none of the grey dots appear as low on the y-axis compared with the blue and orange dots corresponding to the other two policies. This indicates that calibration with the medium and near-optimal policy models incorrectly inflate the performance of many hyperparameter combinations, whereas calibration with the random-policy model potentially undervalues some hyperparameter combinations.

One could certainly argue that many applications might not exhibit this behavior---especially since it is largely caused by a task with two modes of operation (failing or balancing). 
%Additionally, using a random policy to achieve coverage is unrealistic in applications like water treatment. 
Extracting a policy initialization from the calibration data (perhaps via behavior cloning) and then using this initial policy in both hyperparameter selection and deployment could avoid these problems in Cartpole, but we leave these experiments to future work. Regardless, this experiment provides an important reminder that we will not anticipate all situations in the deployment of RL in the real world; there is no general black-box strategy for deployment and failures will happen.

%\section{Improving Calibration Model Efficiency: Integrate with Automatic Hyperparameter Tuning}\label{sec_cem}
\section{Moving Beyond Grid Search}\label{sec_cem}

The calibration model is an offline artifact that we can use as we like without impacting the deployment environment. We can use the model in smarter ways to discover and evaluate good hyperparameters for deployment. In fact, we can leverage the large literature on \emph{algorithm configuration}, which explicitly deals with efficient methods to search over hyperparameters. In this section, we explain how to incorporate these approaches and test two strategies, as a replacement for grid search.  

\subsection{Improving the Hyperparameter Search}

A variety of hyperparameter search approaches were introduced under sequential model-based optimization (SMBO) \citep{hutter2011sequential}, but methods built on Bayesian optimization (BO) \citep{snoek2012practical} have become more popular. Complementary to these approaches are those that direct computation to promising hyperparameters and stop performance evaluations of poor hyperparameters early, as in Hyperband \citep{li2018hyperband}, or that design the algorithm to do both \citep{klein2017fast,falkner2018bohb}. All these approaches attempt to find the maximum of the performance function, assuming that function is expensive to query. 

BO algorithms approximate the performance function $f(\lambda)$, and use this approximation to decide what hyperparameter setting $\lambda$ to test next. The general approach is to (1) maintain a posterior distribution over the performance function $f$, (2) find a candidate set of optimal hyperparameters $\lambda_c$ according to a criteria like expected improvement under the current posterior over $f$, (3) evaluate $\lambda_c$, obtaining $y = f(\lambda_c)$ and (4) update the posterior with sample $(\lambda_c, y)$. Once the algorithm terminates---typically by reaching a time limit---the best $\lambda_c$ out of all the candidates tested is returned, according to the maximal $y$. The primary purpose of learning the posterior $f$ is to direct which $\lambda_c$ should be tested, though some algorithms do solve an optimization at the very end of this procedure to find $\lambda$ with the maximal posterior mean (see \citep{frazier2018tutorial} for a nice overview).   

Due to the importance of hyperparameter optimization for machine learning---in a growing field called Auto ML---the development of BO methods has been focused on large numbers of hyperparameters, for training large models on large datasets. For this highly expensive setting, it is worth carefully crafting advanced algorithms that minimize the need to train and evaluate large models. These complex methods can then be released within packages, to facilitate their use, as they may be difficult to implement from scratch. 

\textbf{BO in our experiments:} We use an open-source package \citep{bayesopt}, which uses gaussian processes for optimizing the hyperparameter setting. We chose to use upper confidence bounds, with a confidence level of 2.576---the default in the package---as the acquisition method. The queue is initialized with 5 random samples and the algorithm is run for 200 iterations.

For our setting, each evaluation is not as complex and we need not use such advanced approaches. Instead, our primary goal is simply to answer: if we allow hyperparameters to be optimized over a continuous set, can we improve on a basic grid search? For this question, we also test two simple approaches: random search and the cross-entropy method (CEM). Random search involves simply testing a fixed number of hyperparameter settings, and selecting the one with maximal performance. Though simple, it is a common baseline because it has been shown to be competitive with more complex search strategies \citep{bergstra2012random}. 

CEM \citep{rubinstein1999crossentropy} is an approach to global optimization of functions, like BO, but is based on a simpler strategy of (1) maintaining a distribution over the inputs (hyperparameters), (2) increasing the likelihood of the top percentile of sampled values under this distribution, according to the performance function. The distribution is simple to sample, and the percentile easy-to-compute, making this approach simpler to use than BO. 

CEM has not been used for hyperparameter optimization, to the best of our knowledge. Likely the reason is that BO strategies provide a more powerful way to carefully reason about what candidate points to sample. CEM instead slowly narrows a distribution over hyperparameters, and does not reason about confidence intervals nor about a criterion (acquisition function) to identify ideal candidates. 
Nonetheless, we include it as a simpler strategy, to investigate performance of using calibration models with a hyperparameter search in-between naive random search and the more advanced BO search. 
%It is natural, therefore, to build on BO and consistently improve those methods, rather than switch to a less developed approach for hyperparameter optimization. For our purposes, however, because each performance evaluation is not as expensive as in other areas of machine learning, CEM is a suitable choice, and much simpler to use. 
We emphasize that it is not critical which hyperparameter optimization approach is used within our framework; any method can be swapped in. 

%\subsection{CEM for More Efficient Hyperparameter Search}
%CEM is a family of algorithms, where the general principle is to 1) maintain a distribution over potential maxima and 2) sample from that distribution and progressively concentrate around the maximal input. This blackbox optimizer allows us to optimize our function---the performance of a hyperparameter $\lambda$ in the calibration model---which is not differentiable. 

\textbf{CEM in our experiments:} Our setting has two nuances compared to the typical setting where CEM is used: our function is expensive to query and we only get a stochastic sample. We provide a modified CEM algorithm, that still reflects the general framework for CEM, but using an incremental update---similar to stochastic gradient descent---to account for the stochasticity in our function query. The algorithm is summarized in Algorithm \ref{alg:AltCEM} in Appendix \ref{app_cem}.

%\subsection{ Experiment 5: Hyperparameter Tuning with CEM}
\subsection{ Experiment 5: Hyperparameter Tuning with Alternative Optimization Approaches}

In this section we compare grid search, random search, Bayesian optimization, and our simple CEM approach for hyperparameter selection. Are these approaches interchangeable in our setup? Does searching a continuous hyperparameter space result in performance gains and at what cost? The goal of the experiment is to highlight that alternative hyperparameter optimization approaches beyond a basic grid search are possible and to investigate if they are beneficial.

In this experiment, we use the same settings as above, but now optimize the temperature $\tau$ and stepsize $\alpha$ as continuous values in the ranges [0.0001, 5.0] and (0.0, 0.1] respectively for Acrobot, and [0.0001, 10.0] and [0.0, 1.0] respectively for Puddle World. The random search approach simply generates $k$ possible hyperparameters from the continuous ranges above and evaluates each in parallel in the calibration model. The best performing hyperparameters according to the calibration phase are used in deployment. Both random search and CEM use 100 iterations, to make computation used comparable to grid search, while Bayesian optimization uses 200 iterations.
%CEM can search this continuous space efficiently, because it uses a truncated multivariate normal that allows for generalization about the utility of hyperparameter values. 

%\begin{wrapfigure}[20]{L}{0.4\textwidth}
\begin{figure*} [t]
	\begin{centering}
%		\vspace{-1cm}
		\includegraphics[width=0.8\textwidth]{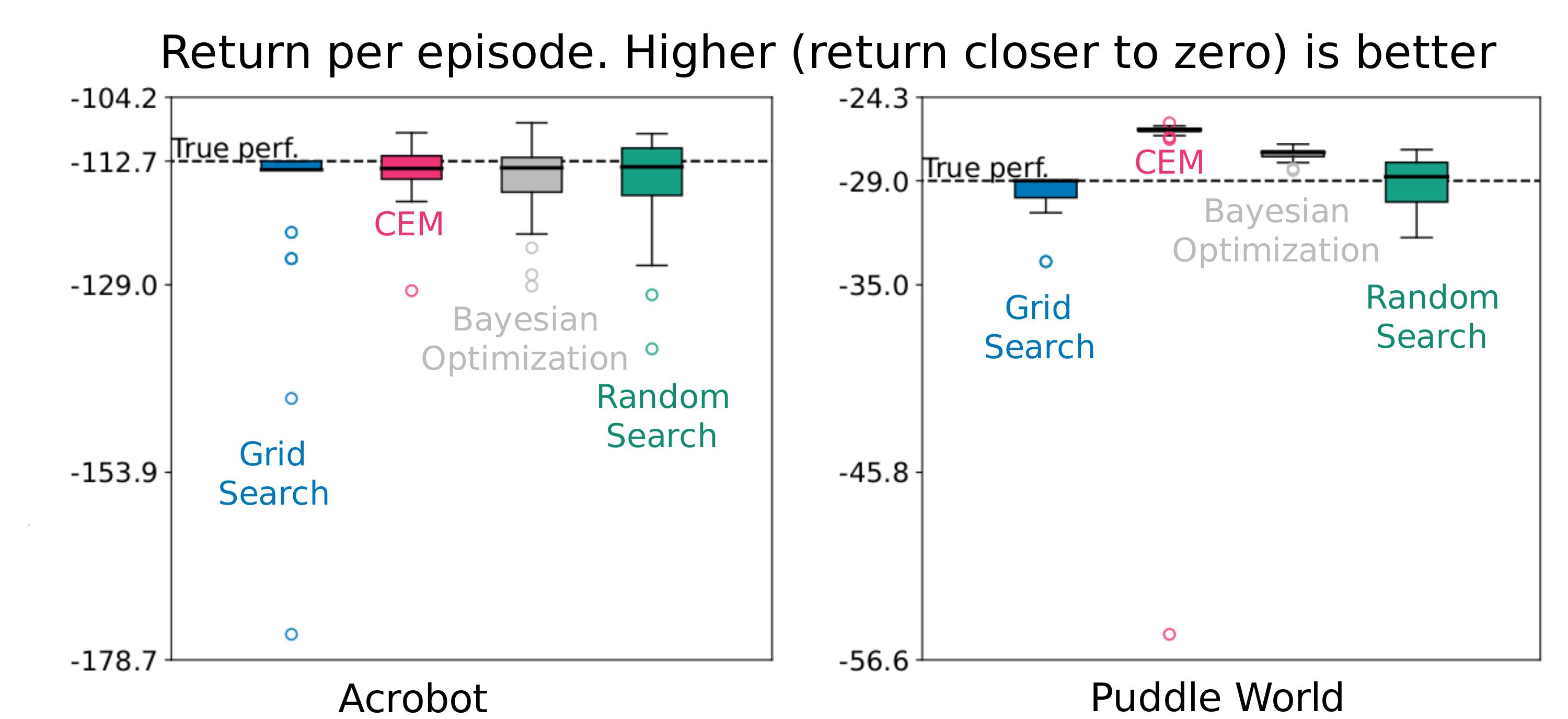}
		\caption{\textbf{Finding good hyperparameters via black-box optimization in calibration.} Here we compare four different strategies for optimizing hyperparameters: (1) grid search (which we have used in all previous experiments),  (2) random search, (3) our CEM procedure, and (4) the Bayesian optimization (BO) approach. The y-axis is same as the one in Figure \ref{fig:model_works}. Generally, all four approaches perform well highlighting the robustness of using the calibration model and transferring hyperparameters to deployment. Random search is better or comparable to grid search, but this is not a surprising as random search is typically found to be a strong baseline in hyperparameter optimization. Note the {\em True performance} in the plots above represent the deployment performance of the best hyperparameter combination from a discrete set; the same set used by grid search. CEM, BO and random search can obtain higher deployment performance because they search a continuous range and thus find better hyperparameter settings. This is one of the major benefits of using better optimization methods in calibration than grid search.}
		\label{fig:exp_cem}
	\end{centering}	
\end{figure*}
%\end{wrapfigure}

Random search, Bayesian optimization and CEM outperform grid search, as we can see in Figure \ref{fig:exp_cem}.
The performance improvements are especially stark in Puddle World. Even when tuning only on the calibration model, the agent can outperform the best hyperparameters found by a grid search on the true environment. This is why the return for CEM is higher than the dotted line showing the performance of the best hyperparameters within the set used for the grid search. These results are promising, in that they show more carefully optimizing hyperparameters on the calibration model helps rather than hurts. A possible hypothesis, apriori, could have been that optimizing more carefully to the calibration model could cause odd or very poor hyperparameters to be chosen and that the restricted set in the grid search actually helped avoid this failure. These results suggest otherwise, and in fact highlight that our previous results could have been produced even more consistent performance with a smarter hyperparameter algorithm.  

This experiments in this paper highlight the generality and flexibility of our approach. The calibration model can take different forms. Data collection can be done in a number of different ways. Hyperparameters can be systematically searched or optimized. In the end, numerous other specializations and algorithmic innovations could be included to further optimize performance in real deployment scenarios.

\section{Conclusion}\label{sec:conclusion}
In this work, we introduced the Data2Online problem: selecting hyperparameters from log of data, for deployment in a real environment. %In real-world systems there is only once change to deploy the agent, and so it is critical to have an offline approach to fully specify the agent for online learning and control in the real world. 
The basic idea is to learn a calibration model from the data log, and then allow the agent to interact in the calibration model to identify good hyperparameters. Essentially, the calibration model is treated just like the real environment. We provide a simple approach, using k-nearest neighbors, to obtain a calibration model that is stable under many iterations and only produces real states. We then conduct a battery of tests, under different data regimes. 
% this We provided a mechanism to learn a calibration model that is st where there exists only one chance to deploy the real-world agent. The calibration model is designed to give a stable and realistic prediction even after a long rollout, while it is not necessary for the calibration model to be completely the same as the true environment.

Naturally, as the first work explicitly tackling this problem, we have only scratched the surface of options. 
There is much more to understand about when this strategy will be effective, and when it might fail. 
As we highlight throughout, this problem should be more feasible than offline RL, which requires the entire policy to be identified from a log rather than just suitable hyperparameters for learning. Our own experiments highlight that offline methods that attempted to learn and deploy a fixed policy performed poorly, whereas identifying reasonable hyperparameters was a much easier problem with consistently good performance across many policies and even small datasets. Nonetheless, we did identify one failure case, where the data resulted in a model that made the environment appear too easy and so most hyperparameters looked similar. Much more work can be done, theoretically and empirically, to understand the Data2Online problem.

%Our experiments suggest that if an agent uses a hyperparameter chosen by the calibration model, learning from scratch in the true environment can be better than transferring a fixed policy. 
%%The calibration model works with multiple agents and function approximators. 
%With a better hyperparameter selection strategy such as CEM, the calibration model helps with choosing a better hyperparameter than grid search. When applying the calibration model in practice, we suggest paying attention to the coverage of the data log; since both the terminal states and reward information come from the offline data log, missing information may cause a worse hyperparameter selection result with the calibration model.

%\subsubsection*{Broader Impact Statement}
%In this optional section, TMLR encourages authors to discuss possible repercussions of their work,
%notably any potential negative impact that a user of this research should be aware of. 
%Authors should consult the TMLR Ethics Guidelines available on the TMLR website
%for guidance on how to approach this subject.

%\subsubsection*{Author Contributions}
%If you'd like to, you may include a section for author contributions as is done
%in many journals. This is optional and at the discretion of the authors. Only add
%this information once your submission is accepted and deanonymized. 

%\subsubsection*{Acknowledgments}
%We thank Nathan Taylor for helpful discussions. This work was supported by NSERC, the Alberta Machine Intelligence Institute (Amii) and CIFAR, particularly through the Canadian AI CIFAR Chairs (CCAI) program. 

\bibliography{tmlr}

\begin{thebibliography}{66}
\providecommand{\natexlab}[1]{#1}
\providecommand{\url}[1]{\texttt{#1}}
\expandafter\ifx\csname urlstyle\endcsname\relax
  \providecommand{\doi}[1]{doi: #1}\else
  \providecommand{\doi}{doi: \begingroup \urlstyle{rm}\Url}\fi

\bibitem[Abbas et~al.(2020)Abbas, Sokota, Talvitie, and
  White]{abbas2020selective}
Zaheer Abbas, Samuel Sokota, Erin Talvitie, and Martha White.
\newblock Selective dyna-style planning under limited model capacity.
\newblock In \emph{International Conference on Machine Learning}, 2020.

\bibitem[Ajay et~al.(2021)Ajay, Kumar, Agrawal, Levine, and
  Nachum]{ajay2020opal}
Anurag Ajay, Aviral Kumar, Pulkit Agrawal, Sergey Levine, and Ofir Nachum.
\newblock Opal: Offline primitive discovery for accelerating offline
  reinforcement learning.
\newblock In \emph{International Conference on Learning Representations}, 2021.

\bibitem[Barde et~al.(2020)Barde, Roy, Jeon, Pineau, Nowrouzezahrai, and
  Pal]{barde2020adversarial}
Paul Barde, Julien Roy, Wonseok Jeon, Joelle Pineau, Derek Nowrouzezahrai, and
  Christopher Pal.
\newblock Adversarial soft advantage fitting: Imitation learning without policy
  optimization.
\newblock In \emph{Advances in Neural Information Processing Systems}, 2020.

\bibitem[Behbahani et~al.(2019)Behbahani, Shiarlis, Chen, Kurin, Kasewa,
  Stirbu, Gomes, Paul, Oliehoek, Messias, et~al.]{behbahani2019learning}
Feryal Behbahani, Kyriacos Shiarlis, Xi~Chen, Vitaly Kurin, Sudhanshu Kasewa,
  Ciprian Stirbu, Joao Gomes, Supratik Paul, Frans~A Oliehoek, Joao Messias,
  et~al.
\newblock Learning from demonstration in the wild.
\newblock In \emph{International Conference on Robotics and Automation}, 2019.

\bibitem[Bentley(1975)]{bentley1975multidimensional}
Jon~Louis Bentley.
\newblock Multidimensional binary search trees used for associative searching.
\newblock \emph{Communications of the ACM}, 18\penalty0 (9):\penalty0 509--517,
  1975.

\bibitem[Bergstra \& Bengio(2012)Bergstra and Bengio]{bergstra2012random}
James Bergstra and Yoshua Bengio.
\newblock Random search for hyper-parameter optimization.
\newblock \emph{Journal of Machine Learning Research}, 13:\penalty0 281--305,
  2012.

\bibitem[Ceron \& Castro(2021)Ceron and Castro]{obando2020revisiting}
Johan Samir~Obando Ceron and Pablo~Samuel Castro.
\newblock Revisiting {Rainbow}: Promoting more insightful and inclusive deep
  reinforcement learning research.
\newblock In \emph{International Conference on Machine Learning}, 2021.

\bibitem[Chandak et~al.(2020{\natexlab{a}})Chandak, Jordan, Theocharous, White,
  and Thomas]{chandak2020towards}
Yash Chandak, Scott~M Jordan, Georgios Theocharous, Martha White, and Philip~S
  Thomas.
\newblock Towards safe policy improvement for non-stationary mdps.
\newblock In \emph{Advances in Neural Information Processing Systems},
  2020{\natexlab{a}}.

\bibitem[Chandak et~al.(2020{\natexlab{b}})Chandak, Theocharous, Shanka, White,
  Mahadevan, and Thomas]{chandak2020optimizing}
Yash Chandak, Georgios Theocharous, Shiv Shanka, Martha White, Sridhar
  Mahadevan, and Philip~S Thomas.
\newblock Optimizing for the future in non-stationary mdps.
\newblock In \emph{International Conference on Machine Learning},
  2020{\natexlab{b}}.

\bibitem[Chelu et~al.(2020)Chelu, Precup, and van
  Hasselt]{chelu2020forethought}
Veronica Chelu, Doina Precup, and Hado~P van Hasselt.
\newblock Forethought and hindsight in credit assignment.
\newblock In \emph{Advances in Neural Information Processing Systems}, 2020.

\bibitem[Downey \& Sanner(2010)Downey and Sanner]{downey2010temporal}
C.~Downey and S.~Sanner.
\newblock Temporal {difference} {Bayesian} {model} {averaging}: {A} {Bayesian}
  {perspective} on {adapting} {Lambda}.
\newblock In \emph{{International Conference on Machine Learning}}, 2010.

\bibitem[Ernst et~al.(2005)Ernst, Geurts, and Wehenkel]{ernst2005treebased}
Damien Ernst, Pierre Geurts, and Louis Wehenkel.
\newblock Tree-{based} {batch} {mode} {reinforcement} {learning}.
\newblock \emph{Journal of Machine Learning Research}, 6:\penalty0 503--556,
  2005.

\bibitem[Falkner et~al.(2018)Falkner, Klein, and Hutter]{falkner2018bohb}
Stefan Falkner, A.~Klein, and F.~Hutter.
\newblock {BOHB}: {Robust} and {efficient} {hyperparameter} {optimization} at
  {scale}.
\newblock In \emph{{International Conference on Machine Learning}}, 2018.

\bibitem[Finn et~al.(2017)Finn, Yu, Zhang, Abbeel, and Levine]{finn2017one}
Chelsea Finn, Tianhe Yu, Tianhao Zhang, Pieter Abbeel, and Sergey Levine.
\newblock One-shot visual imitation learning via meta-learning.
\newblock In \emph{Conference on Robot Learning}, 2017.

\bibitem[Frazier(2018)]{frazier2018tutorial}
Peter~I Frazier.
\newblock A tutorial on bayesian optimization.
\newblock \emph{arXiv preprint arXiv:1807.02811}, 2018.

\bibitem[Ghasemipour et~al.(2019)Ghasemipour, Zemel, and
  Gu]{ghasemipour2020divergence}
Seyed Kamyar~Seyed Ghasemipour, Richard Zemel, and Shixiang Gu.
\newblock A divergence minimization perspective on imitation learning methods.
\newblock In \emph{Conference on Robot Learning}, 2019.

\bibitem[Henderson et~al.(2018)Henderson, Islam, Bachman, Pineau, Precup, and
  Meger]{henderson2018deep}
Peter Henderson, Riashat Islam, Philip Bachman, Joelle Pineau, Doina Precup,
  and David Meger.
\newblock Deep reinforcement learning that matters.
\newblock In \emph{the AAAI Conference on Artificial Intelligence}, 2018.

\bibitem[Higgins et~al.(2017)Higgins, Pal, Rusu, Matthey, Burgess, Pritzel,
  Botvinick, Blundell, and Lerchner]{pmlr-v70-higgins17a}
Irina Higgins, Arka Pal, Andrei Rusu, Loic Matthey, Christopher Burgess,
  Alexander Pritzel, Matthew Botvinick, Charles Blundell, and Alexander
  Lerchner.
\newblock {DARLA}: Improving zero-shot transfer in reinforcement learning.
\newblock In \emph{International Conference on Machine Learning}, 2017.

\bibitem[Hutter et~al.(2011)Hutter, Hoos, and
  Leyton-Brown]{hutter2011sequential}
Frank Hutter, Holger~H Hoos, and Kevin Leyton-Brown.
\newblock Sequential model-based optimization for general algorithm
  configuration.
\newblock In \emph{International Conference on Learning and Intelligent
  Optimization}, 2011.

\bibitem[Jacobsen et~al.(2019)Jacobsen, Schlegel, Linke, Degris, White, and
  White]{jacobsen2019meta}
Andrew Jacobsen, Matthew Schlegel, Cam Linke, Thomas Degris, Adam White, and
  Martha White.
\newblock Meta-{descent} for {online}, {continual} {prediction}.
\newblock In \emph{the AAAI Conference on Artificial Intelligence}, 2019.

\bibitem[Jaderberg et~al.(2017)Jaderberg, Dalibard, Osindero, Czarnecki,
  Donahue, Razavi, Vinyals, Green, Dunning, Simonyan, Fernando, and
  Kavukcuoglu]{jaderberg2017population}
Max Jaderberg, Valentin Dalibard, Simon Osindero, Wojciech~M Czarnecki, Jeff
  Donahue, Ali Razavi, Oriol Vinyals, Tim Green, Iain Dunning, Karen Simonyan,
  Chrisantha Fernando, and Koray Kavukcuoglu.
\newblock Population based training of neural networks.
\newblock \emph{arXiv preprint arXiv:1711.09846}, 2017.

\bibitem[Jafferjee et~al.(2020)Jafferjee, Imani, Talvitie, White, and
  Bowling]{jafferjee2020hallucinating}
Taher Jafferjee, Ehsan Imani, Erin Talvitie, Martha White, and Micheal Bowling.
\newblock Hallucinating value: A pitfall of dyna-style planning with imperfect
  environment models.
\newblock \emph{arXiv preprint arXiv:2006.04363}, 2020.

\bibitem[Ke et~al.(2019)Ke, Singh, Touati, Goyal, Bengio, Parikh, and
  Batra]{ke2019learning}
Nan~Rosemary Ke, Amanpreet Singh, Ahmed Touati, Anirudh Goyal, Yoshua Bengio,
  Devi Parikh, and Dhruv Batra.
\newblock Learning dynamics model in reinforcement learning by incorporating
  the long term future.
\newblock In \emph{International Conference on Learning Representations}, 2019.

\bibitem[Kearns \& Singh(2002)Kearns and Singh]{kearns2002nearoptimal}
Michael Kearns and Satinder Singh.
\newblock Near-{optimal} {reinforcement} {learning} in {polynomial} {time}.
\newblock \emph{Machine Learning}, 49\penalty0 (2):\penalty0 209--232, 2002.

\bibitem[Kingma \& Ba(2014)Kingma and Ba]{kingma2014adam}
Diederik~P Kingma and Jimmy Ba.
\newblock Adam: A method for stochastic optimization.
\newblock \emph{arXiv preprint arXiv:1412.6980}, 2014.

\bibitem[Klein et~al.(2017)Klein, Falkner, Bartels, Hennig, and
  Hutter]{klein2017fast}
Aaron Klein, Stefan Falkner, Simon Bartels, Philipp Hennig, and Frank Hutter.
\newblock Fast {Bayesian} {optimization} of {machine} {learning}
  {hyperparameters} on {large} {datasets}.
\newblock In \emph{International Conference on Artificial Intelligence and
  Statistics}, 2017.

\bibitem[Lattimore \& Szepesv{\'a}ri(2020)Lattimore and
  Szepesv{\'a}ri]{lattimore2020bandit}
Tor Lattimore and Csaba Szepesv{\'a}ri.
\newblock \emph{Bandit algorithms}.
\newblock Cambridge University Press, 2020.

\bibitem[Lee et~al.(2021)Lee, Seo, Lee, Abbeel, and
  Shin]{lee2021offlinetoonline}
Seunghyun Lee, Younggyo Seo, Kimin Lee, Pieter Abbeel, and Jinwoo Shin.
\newblock Offline-to-{online} {reinforcement} {learning} via {balanced}
  {replay} and {pessimistic} {Q}-{ensemble}.
\newblock In \emph{Conference on Robot Learning}, 2021.

\bibitem[Li et~al.(2018)Li, Jamieson, DeSalvo, Rostamizadeh, and
  Talwalkar]{li2018hyperband}
Lisha Li, Kevin Jamieson, Giulia DeSalvo, Afshin Rostamizadeh, and Ameet
  Talwalkar.
\newblock Hyperband: A novel bandit-based approach to hyperparameter
  optimization.
\newblock \emph{Journal of Machine Learning Research}, 18\penalty0
  (1):\penalty0 6765–6816, 2018.

\bibitem[Mahmood et~al.(2018)Mahmood, Korenkevych, Vasan, Ma, and
  Bergstra]{mahmood2018benchmarking}
A.~Rupam. Mahmood, Dmytro Korenkevych, Gautham Vasan, W.~Ma, and J.~Bergstra.
\newblock Benchmarking {reinforcement} {learning} {algorithms} on
  {real}-{world} {robots}.
\newblock In \emph{{Conference on Robot Learning}}, 2018.

\bibitem[Mandel et~al.(2016)Mandel, Liu, Brunskill, and
  Popovi\'{c}]{mandel2016offline}
Travis Mandel, Yun-En Liu, Emma Brunskill, and Zoran Popovi\'{c}.
\newblock Offline evaluation of online reinforcement learning algorithms.
\newblock In \emph{the AAAI Conference on Artificial Intelligence}, 2016.

\bibitem[Manek \& Kolter(2019)Manek and Kolter]{manek2019learning}
Gaurav Manek and J~Zico Kolter.
\newblock Learning {stable} {deep} {dynamics} {models}.
\newblock In \emph{Advances in Neural Information Processing Systems}, 2019.

\bibitem[Mann et~al.(2016)Mann, Penedones, Mannor, and
  Hester]{mann2016adaptive}
Timothy~A. Mann, Hugo Penedones, Shie Mannor, and T.~Hester.
\newblock Adaptive {Lambda} {least}-{squares} {temporal} {difference}
  {learning}.
\newblock \emph{arXiv preprint arXiv:1612.09465}, 2016.

\bibitem[massoud Farahmand et~al.(2009)massoud Farahmand, Ghavamzadeh,
  Szepesv{\'a}ri, and Mannor]{farahmand2009regularized}
Amir massoud Farahmand, Mohammad Ghavamzadeh, Csaba Szepesv{\'a}ri, and Shie
  Mannor.
\newblock Regularized fitted q-iteration for planning in continuous-space
  markovian decision problems.
\newblock In \emph{American Control Conference}, 2009.

\bibitem[Merel et~al.(2017)Merel, Tassa, TB, Srinivasan, Lemmon, Wang, Wayne,
  and Heess]{merel2017learning}
Josh Merel, Yuval Tassa, Dhruva TB, Sriram Srinivasan, Jay Lemmon, Ziyu Wang,
  Greg Wayne, and Nicolas Heess.
\newblock Learning human behaviors from motion capture by adversarial
  imitation.
\newblock \emph{arXiv preprint arXiv:1707.02201}, 2017.

\bibitem[Nogueira(2014--)]{bayesopt}
Fernando Nogueira.
\newblock {Bayesian Optimization}: Open source constrained global optimization
  tool for {Python}, 2014--.
\newblock URL \url{https://github.com/fmfn/BayesianOptimization}.

\bibitem[Paine et~al.(2020)Paine, Paduraru, Michi, Gulcehre, Zolna, Novikov,
  Wang, and de~Freitas]{paine2020hyperparameter}
Tom~Le Paine, Cosmin Paduraru, Andrea Michi, Caglar Gulcehre, Konrad Zolna,
  Alexander Novikov, Ziyu Wang, and Nando de~Freitas.
\newblock Hyperparameter selection for offline reinforcement learning.
\newblock In \emph{Offline Reinforcement Learning Workshop at Neural
  Information Processing Systems}, 2020.

\bibitem[Pan et~al.(2018)Pan, Zaheer, White, Patterson, and White]{pan2017rem}
Yangchen Pan, Muhammad~Hamad Zaheer, Adam White, Andrew Patterson, and Martha
  White.
\newblock Organizing experience: A deeper look at replay mechanisms for
  sample-based planning in continuous state domains.
\newblock In \emph{International Joint Conference on Artificial Intelligence},
  2018.

\bibitem[Papini et~al.(2019)Papini, Pirotta, and Restelli]{papini2019smoothing}
Matteo Papini, Matteo Pirotta, and Marcello Restelli.
\newblock Smoothing {policies} and {safe} {policy} {gradients}.
\newblock \emph{arXiv preprint arXiv:1905.03231}, May 2019.

\bibitem[Parker-Holder et~al.(2020)Parker-Holder, Nguyen, and
  Roberts]{parker-holder2020provably}
Jack Parker-Holder, Vu~Nguyen, and Stephen~J. Roberts.
\newblock Provably efficient online hyperparameter optimization with
  population-based bandits.
\newblock In \emph{Conference and Workshop on Neural Information Processing
  Systems}, 2020.

\bibitem[Paul et~al.(2019)Paul, Kurin, and Whiteson]{paul2019fast}
Supratik Paul, Vitaly Kurin, and Shimon Whiteson.
\newblock Fast efficient hyperparameter tuning for policy gradient methods.
\newblock In \emph{Advances in Neural Information Processing Systems}, 2019.

\bibitem[Peng et~al.(2018)Peng, Andrychowicz, Zaremba, and Abbeel]{peng2018sim}
Xue~Bin Peng, Marcin Andrychowicz, Wojciech Zaremba, and Pieter Abbeel.
\newblock Sim-to-real transfer of robotic control with dynamics randomization.
\newblock In \emph{International Conference on Robotics and Automation}, 2018.

\bibitem[Ravichandar et~al.(2020)Ravichandar, Polydoros, Chernova, and
  Billard]{ravichandar2020recent}
Harish Ravichandar, Athanasios~S Polydoros, Sonia Chernova, and Aude Billard.
\newblock Recent advances in robot learning from demonstration.
\newblock \emph{Annual Review of Control, Robotics, and Autonomous Systems},
  3:\penalty0 297--330, 2020.

\bibitem[Riedmiller(2005)]{riedmiller2005neural}
Martin Riedmiller.
\newblock Neural {Fitted} {Q} {Iteration} – {First} {experiences} with a
  {data} {efficient} {neural} {reinforcement} {learning} {method}.
\newblock In \emph{European Conference on Machine Learning}, 2005.

\bibitem[Rubinstein(1999)]{rubinstein1999crossentropy}
Reuven Rubinstein.
\newblock The cross-entropy method for combinatorial and continuous
  optimization.
\newblock \emph{Methodology and computing in applied probability}, 1\penalty0
  (2):\penalty0 127--190, 1999.

\bibitem[Snoek et~al.(2012)Snoek, Larochelle, and Adams]{snoek2012practical}
Jasper Snoek, Hugo Larochelle, and Ryan~P. Adams.
\newblock Practical {Bayesian} optimization of machine learning algorithms.
\newblock In \emph{Advances in Neural Information Processing Systems}, 2012.

\bibitem[Srinivas et~al.(2010)Srinivas, Krause, Kakade, and
  Seeger]{srinivas2010gaussian}
Niranjan Srinivas, Andreas Krause, Sham Kakade, and Matthias Seeger.
\newblock Gaussian process optimization in the bandit setting: no regret and
  experimental design.
\newblock In \emph{International Conference on Machine Learning}, 2010.

\bibitem[Sutton(1992)]{sutton1992adapting}
Richard~S Sutton.
\newblock Adapting bias by gradient descent: An incremental version of
  delta-bar-delta.
\newblock In \emph{the AAAI Conference on Artificial Intelligence}, 1992.

\bibitem[Sutton(1995)]{sutton1996generalization}
Richard~S Sutton.
\newblock Generalization in reinforcement learning: Successful examples using
  sparse coarse coding.
\newblock In \emph{Advances in Neural Information Processing Systems}, 1995.

\bibitem[Sutton \& Barto(2018)Sutton and Barto]{sutton2018reinforcement}
Richard~S Sutton and Andrew~G Barto.
\newblock \emph{Reinforcement learning: An introduction}.
\newblock MIT press, 2018.

\bibitem[Sutton et~al.(2007)Sutton, Koop, and Silver]{sutton2007role}
Richard~S Sutton, Anna Koop, and David Silver.
\newblock On the role of tracking in stationary environments.
\newblock In \emph{International Conference on Machine learning}, 2007.

\bibitem[Talvitie(2017)]{talvitie2017self}
Erik Talvitie.
\newblock Self-correcting models for model-based reinforcement learning.
\newblock In \emph{the AAAI Conference on Artificial Intelligence}, 2017.

\bibitem[Talvitie(2014)]{talvitie2014model}
Erin Talvitie.
\newblock Model regularization for stable sample rollouts.
\newblock In \emph{International Conference on Uncertainty in Artificial
  Intelligence}, 2014.

\bibitem[Tang \& Choromanski(2020)Tang and Choromanski]{tang2020online}
Yunhao Tang and Krzysztof Choromanski.
\newblock Online {hyper}-parameter {tuning} in {off}-policy {learning} via
  {evolutionary} {strategies}.
\newblock \emph{arXiv preprint arXiv:2006.07554}, 2020.

\bibitem[Venkatraman et~al.(2015)Venkatraman, Hebert, and
  Bagnell]{venkatraman2015improving}
Arun Venkatraman, Martial Hebert, and J~Andrew Bagnell.
\newblock Improving multi-step prediction of learned time series models.
\newblock In \emph{the AAAI Conference on Artificial Intelligence}, 2015.

\bibitem[Wang et~al.(2021)Wang, Foster, and Kakade]{wang2020what}
Ruosong Wang, Dean~P. Foster, and Sham~M. Kakade.
\newblock What are the statistical limits of offline {RL} with linear function
  approximation?
\newblock In \emph{International Conference on Learning Representations}, 2021.

\bibitem[White(2017)]{white2017unifying}
Martha White.
\newblock Unifying task specification in reinforcement learning.
\newblock In \emph{International {Conference} on {Machine} {Learning}}, 2017.

\bibitem[White \& White(2016)White and White]{white2016greedy}
Martha White and Adam White.
\newblock A {greedy} {approach} to {adapting} the {trace} {parameter} for
  {temporal} {difference} {learning}.
\newblock In \emph{International Conference on Autonomous Agents and Multiagent
  Systems}, 2016.

\bibitem[Williams et~al.(2017)Williams, Wagener, Goldfain, Drews, Rehg, Boots,
  and Theodorou]{williams2017information}
Grady Williams, Nolan Wagener, Brian Goldfain, Paul Drews, James~M Rehg, Byron
  Boots, and Evangelos~A Theodorou.
\newblock Information theoretic {MPC} for model-based reinforcement learning.
\newblock In \emph{International Conference on Robotics and Automation}, 2017.

\bibitem[Wu et~al.(2019{\natexlab{a}})Wu, Tucker, and Nachum]{wu2018laplacian}
Yifan Wu, George Tucker, and Ofir Nachum.
\newblock The laplacian in rl: Learning representations with efficient
  approximations.
\newblock In \emph{International Conference on Learning Representations},
  2019{\natexlab{a}}.

\bibitem[Wu et~al.(2019{\natexlab{b}})Wu, Tucker, and Nachum]{wu2019behavior}
Yifan Wu, George Tucker, and Ofir Nachum.
\newblock Behavior regularized offline reinforcement learning.
\newblock \emph{arXiv preprint arXiv:1911.11361}, 2019{\natexlab{b}}.

\bibitem[Xing et~al.(2021)Xing, Nagata, Chen, Zou, Neftci, and
  Krichmar]{xing2021domain}
Jinwei Xing, Takashi Nagata, Kexin Chen, Xinyun Zou, Emre Neftci, and
  Jeffrey~L. Krichmar.
\newblock Domain adaptation in reinforcement learning via latent unified state
  representation.
\newblock In \emph{the AAAI Conference on Artificial Intelligence}, 2021.

\bibitem[Xu et~al.(2018)Xu, van Hasselt, and Silver]{xu2018metagradient}
Zhongwen Xu, Hado van Hasselt, and David Silver.
\newblock Meta-gradient reinforcement learning.
\newblock In \emph{Advances in Neural Information Processing Systems}, 2018.

\bibitem[Yang \& Nachum(2021)Yang and Nachum]{yang2021representation}
Mengjiao Yang and Ofir Nachum.
\newblock Representation {matters}: {Offline} {pretraining} for {sequential}
  {decision} {making}.
\newblock In \emph{{International} {Conference} on {Machine} {Learning}}, 2021.

\bibitem[Yang et~al.(2020)Yang, Dai, Nachum, Tucker, and
  Schuurmans]{yang2020offline}
Mengjiao Yang, Bo~Dai, Ofir Nachum, George Tucker, and Dale Schuurmans.
\newblock Offline policy selection under uncertainty.
\newblock \emph{arXiv preprint arXiv:2012.06919}, 2020.

\bibitem[Zahavy et~al.(2020)Zahavy, Xu, Veeriah, Hessel, Oh, van Hasselt,
  Silver, and Singh]{zahavy2020self}
Tom Zahavy, Zhongwen Xu, Vivek Veeriah, Matteo Hessel, Junhyuk Oh, Hado~P van
  Hasselt, David Silver, and Satinder Singh.
\newblock A self-tuning actor-critic algorithm.
\newblock In \emph{Advances in Neural Information Processing Systems}, 2020.

\end{thebibliography}
\bibliographystyle{tmlr}

\newpage
\appendix
\section{Algorithm Details}\label{apdx:details}

In this section, we explain the calibration model, agents used in experiments, and baselines in detail then provide more detailed pseudocode.

\subsection{Calibration Model} \label{apdx:calibration-detail}
Below, Algorithm \ref{treeconstr} explains tree construction in discrete and continuous action spaces, and Algorithm \ref{simulatorstart} indicates the process of sampling the start state in each episode. The initialization of calibration model has been indicated in Algorithm \ref{alg_knn_calib}, and the step function has been shown in Algorithm \ref{alg_knn_sample}.

%\begin{algorithm}[htb]	
%	\caption{CalibrationModelTraining}\label{simulatorinit}	
%	\textbf{Input:}
%	$\dataset$: the offline dataset\\
%	\begin{algorithmic}[1]
%		%		\STATE $S, A, S', R, T \leftarrow Format(\dataset)$
%		\State Assume $\dataset$ consists of tuples of the form $(S, A, S', R, T )$
%		\State Collect start states $\mathcal{B} \subset \dataset$
%		\State $\psi \leftarrow LaplaceRepTraining(S, A, S', T)$
%		\State $\Phi \leftarrow \psi(S)$
%		\State $Trees \leftarrow KDTreeConstruction(\Phi, A, S', R, T)$
%		\State Fit $Trees$ into the calibration model $\cmodel$
%	\end{algorithmic}
%	\textbf{Return:} $\cmodel$
%\end{algorithm}
\begin{algorithm}[htb]	
	\caption{KDTreeConstruction}\label{treeconstr}	
	\textbf{Input:}
	Transitions with tuples of the form $(\Phi, A, S', R, T)$
	\begin{algorithmic}[1]
		\IF {discrete action space}
		\STATE Construct KD-trees $Trees$ for each action: use $\Phi$ as key and $S', R, T$ as value.
		\ELSE
		\STATE Construct KD-trees $Trees$: use $(\Phi, A)$ as key and $S', R, T$ as value.
		\ENDIF
	\end{algorithmic}
	\textbf{Return:} $Trees$
\end{algorithm}

\begin{algorithm}[htb]	
	\caption{EnvStart (CalibrationModel)}\label{simulatorstart}
	\textbf{Input:}
	A set of start states $\mathcal{B}$
	\begin{algorithmic}[1]
		\STATE Randomly pick $s \in \mathcal{B}$
	\end{algorithmic}
	\textbf{Return:} $s_i$
\end{algorithm}

%\begin{algorithm}[htb]	
%	\caption{EnvStep (CalibrationModel)}\label{simulatorstep}
%	\textbf{Input:} State $s_t$, Action $a$; if no action is given, procedure returns a start state
%	\begin{algorithmic}
%		\If {No action is given}
%		\State\Return Sample $s \in \mathcal{B}$ uniform randomly
%		\EndIf
%		\State // Find $k$ nearest neighbors to $s_t,a_t$, to get potential next states and rewards
%		\State ${ (s'_i, r_i, d_i)}_{i=1}^k \leftarrow$ KDTreeSearch$(\psi(s_t), a, Trees, k)$ in Algorithm \ref{treesearch} 
%		\State // If closest neighbor is far, then return a default return and terminate
%		\If{$\min_i d >$ threshold}
%		\State $r \leftarrow R_{\text{default}}$,  $s' \leftarrow \text{terminal}$
%		\State \Return $(r, s')$
%		\EndIf
%		\State  Sample $ i \in [1,k]$ according to probabilities $p_i = 1- \frac{d_i}{\Sigma_{j \in [1,k]} d_j}$
%		\State \Return $(r_i, s'_i)$
%	\end{algorithmic}
%\end{algorithm}

\begin{algorithm}[htb]	
	\caption{KDTreeSearch}\label{treesearch}	
	\textbf{Input:} Representation $\phi$, Action $a$, KD-tree(s) $Trees$, number of nearest neighbors $k$.
	\begin{algorithmic}[1]
		\IF {discrete action space}
		\STATE $\{(\phi, s', r, T, d)\} \leftarrow$ Search for the k nearest neighbors of $\phi$ in the tree corresponding to the action $a$ ($Trees[a]$), where $d$ refers to the distances.
		\ELSE
		\STATE $\{(\phi, s', r, T, d)\} \leftarrow$ Search for the k nearest neighbors of $(\phi, a)$ in the tree ($Trees$), where $d$ refers to the distances.
		\ENDIF
	\end{algorithmic}
	\textbf{Return:} $\{(s', r, T, d)\}$
\end{algorithm}

In discrete action space, we construct a k-nearest-neighbor dictionary in advance then search in table at every time step to speed-up learning. In this case, Algorithm \ref{alg_knn_calib} and Algorithm \ref{alg_knn_sample} are replaced by Algorithm \ref{alg_knn_calib_faster} and \ref{alg_knn_sample_faster} separately.

\begin{algorithm}[htb]
	\caption{Learn KNN Calibration Model}\label{alg_knn_calib_faster}
	\textbf{Input:} dataset $\dataset$ with tuples of the form $(S, A, S', R)$, number of nearest neighbors $k$ (default $k = 3$) \\
	\textbf{Constructs:} Representation $\psi$ for distances, k-nearest-neighbor dictionary $Neighbors$ for fast nearest neighbors search, $R_{\text{default}}$ default return
	\begin{algorithmic}
		\STATE $\psi \leftarrow $LaplaceRepTraining$(\dataset)$ in Algorithm \ref{reptraining}
		\STATE $\mathcal{S}_{key} = list(\mathcal{B}\cup \{s'|s' \in \dataset\})$
		\STATE Reorder $\dataset$ array to follow the states' order in the list $\mathcal{S}_{key}$
		\STATE $Trees \leftarrow $KDTreeConstruction$(\psi(\mathcal{S}_{key}), \dataset)$ in Algorithm \ref{treeconstr}
		\STATE $Neighbors \leftarrow $KNNTableConstruction$(\psi, \dataset, k, Trees)$ in Algorithm \ref{knntableconstr}
		\STATE Extract starting states $\mathcal{B} \subseteq \mathcal{D}$
		\STATE Set $R_{\text{default}}$ to minimum return in dataset (pessimistic default return)
	\end{algorithmic}
\end{algorithm}

\begin{algorithm}[htb]
	\caption{KNNTableConstruction}\label{knntableconstr}
	\textbf{Input:} Representation function $\psi$, keys $\mathcal{S}_{key}$, dataset $\dataset$, number of nearest neighbors $k$ (default $k = 3$), k-d tree for each action $trees$ \\
	\textbf{Constructs:} k-nearest-neighbor dictionary $Neighbors$
	\begin{algorithmic}
		\STATE $Neighbors \leftarrow []$
		\FOR {$a \in \mathcal{A}$}
			\STATE $Neighbors.append([])$
			\FOR {$s \in \mathcal{S}_{key}$}
			\STATE $knns \leftarrow KDTreeSearch(\psi(s), a, trees, k)$
			\STATE $entry \leftarrow []$
			\FOR {$(s',r,T,d) \in knns$}
				\STATE $entry.append((s',r,T,d,\mathcal{S}_{key}.index(s')))$
			\ENDFOR
			\STATE $Neighbors[a].append(entry)$
			\ENDFOR
		\ENDFOR
	\end{algorithmic}
\end{algorithm}

\begin{algorithm}[t]
	\caption{QuickSample}\label{alg_knn_sample_faster}
	\textbf{Input:} K-nearest-neighbor dictionary $Neighbors$, current state index $index_t$, Action $a$; if no action is given, procedure returns a start state. \\
	\textbf{Global Variable:} The index of the next state in the search table $currentIndex$, initialized by the index of the chosen starting state when each episode starts. \\
	\textbf{Return:} Tuple of the reward and next state $(r_i, s'_i)$.
	\begin{algorithmic}
		\IF {No action is given}
		\RETURN Sample $s \in \mathcal{B}$ uniform randomly
		\ENDIF
		\STATE // Query the table for the $k$ nearest neighbors to $s_t,a_t$, to get potential next states and rewards
		\STATE ${ (s'_i, r_i, d_i, index_i)}_{i=1}^k \leftarrow Neighbors[a_t][currentIndex]$ 
		\STATE // If closest neighbor is far, then return a default return and terminate
		\IF{$\min_i d >$ threshold}
		\STATE $r \leftarrow R_{\text{default}}$,  $s' \leftarrow \text{terminal}$
		\RETURN $(r, s')$
		\ENDIF
		\STATE Sample $ i \in [1,k]$ according to probabilities $p_i = 1- \frac{d_i}{\Sigma_{j \in [1,k]} d_j}$
		\STATE $currentIndex \leftarrow index_i$
		\RETURN $(r_i, s'_i)$
	\end{algorithmic}
\end{algorithm}

\subsubsection{Distance Metrics}\label{app_laplace}
We explain how we learn the Laplace representation in detail in this section.
The Laplace representation is trained by pushing the representations of two random states to be far away from each other, while encouraging the representations of \emph{close} states to be similar. Two states are close if it only takes a few steps for the agent to get to one state from the other. The Laplace representation $\psi(s)$ is trained using a two-layer NN on a batch of data, using the approach given by \citep{wu2018laplacian}. In training, the close state is sampled according to a transition distribution $P_{\kappa}$. When a state $s_0$ in a trajectory $s_0, s_1, s_2, \cdots, s_z$ is chosen, the close state is sampled in the following trajectory $s_1, s_2, \cdots, s_z$ with normalized probability $[\kappa, \kappa^2, \cdots, \kappa^z]$. The loss for training the Laplacian representation is calculated from close state pair $( s_i, s_{i+u} )$ and a random state pair $ (s_i, s_j) $. 
%$$Loss = \frac{1}{2}(\phi_s - \phi_{s_u})^2 + \beta ((\phi_s^T \phi_{s_v})^2 - \zeta \|\phi_s\|^2 - \zeta \|\phi_{s_v} \|^2 )$$
{\footnotesize
	\begin{align*}
	\sum_{s_i \sim \dataset, u \sim P_{\kappa}} \| \psi_\theta(s_i) - \psi_\theta(s_{i+u}) \|_2^2  
	+ \sum_{s_i, s_j \sim \dataset} \left(\beta(\psi_\theta(s_i)^T \psi_\theta(s_{j}))^2 - \zeta\|\psi_\theta(s_i)\|_2^2 - \zeta\|\psi_\theta(s_j)\|_2^2 \right)
	\end{align*}
}
where $\beta$ and $\zeta$ control the weight of each term.

During training, we set the maximum training step as 30,000 steps, the averaged loss is checked every 1000 steps. The representation is considered as converged if the averaged loss over the past 1000 steps does not decrease in 3 consecutive checks, thus we cut off training earlier.

\begin{algorithm}[htb]	
	\caption{LaplaceRepTraining}\label{reptraining}	
	\textbf{Input:} Dataset $\dataset$ with tuples of the form $(S, A, S', R)$, Dataset size $N$, Weights in loss function $\beta, \zeta$, Learning rate$\alpha$, Batch size $b$
	\begin{algorithmic}[1]
		%		\STATE Train the laplacian representation function $RepFunc$ using the method in \ref{wu2018laplacian}
		% \State $Converge \leftarrow False$
		\STATE Initialize representation function $\psi$
		% \State $i \leftarrow 0$
		%		\State $LossList \leftarrow []$
		% \While{not $Converge$ and $i < 30000$} {
		\WHILE{not converge} {
			\STATE $Samples \leftarrow \{s_i\}, i\in [N]$
			\STATE $Closes \leftarrow \{s_{i+u}\}, u \sim P_{\kappa}$
			\STATE $Randoms \leftarrow \{s_j\}, j\in [N]$
			\STATE $\phi_s \leftarrow \psi(Samples)$
			\STATE $\phi_c \leftarrow \psi(Closes)$
			\STATE $\phi_r \leftarrow \psi(Randoms)$
			\STATE Update $\Phi$ with $Loss = (\phi_s - \phi_c)^2 + \beta ((\phi_s^T \phi_r)^2 - \zeta \|\phi_s\|^2 - \zeta \|\phi_r \|^2 )$
			% \State $i \leftarrow i + 1$
			
			% \If {$i \mod 1000 == 0$}
			%			\State $Check_1 \leftarrow Check_2$
			%			\State $Check_2 \leftarrow Check_3$
			%			\State $Check_3 = \Sigma^{m<1000}_{m=0}(LossList[m])$
			%			\State $LossList \leftarrow []$
			%			\If {$Check_1 \leq Check_2 \leq Check_3$}
			%			\State $Converge \leftarrow True$
			%			\EndIf
			% \State $Converge \leftarrow $Check convergence
			%			\State $LossList \leftarrow []$
			%			\Else
			%			\State $LossList.append(Loss)$
			% \EndIf
		}\ENDWHILE
	\end{algorithmic}
	\textbf{Return:} $\psi$
\end{algorithm}

\subsection{Agents}

We describe the agents used in our experiments: Expected Sarsa, Actor-Critic, and Fitted-Q Iteration. The random selection baseline will also be explained at the end.

\subsubsection{Expected Sarsa}
Expected Sarsa is an online learning agent estimating action values. 
We used the \emph{Sarsa($\lambda$)} algorithm from the introductory text \citep{sutton2018reinforcement}, while using Expected Sarsa update and linear function approximator for estimating action values.
The function approximator is parameterized by $\textbf{w}$ and states are projected to binary features by tile coding.
It is updated through TD-error $\delta=R+\gamma\displaystyle\sum_{a}\pi(a \mid S')\hat{q}(S',a,\textbf{w})-\hat{q}(S,A,\textbf{w})$, where $\pi$ refers to the policy and $\gamma$ is the discount rate.
We also used Adam optimizer to learn adaptive step-sizes \citep{kingma2014adam}. 
$\pi$ is a softmax policy such that at timestep $t$, $\pi_t(a \mid S_t) = \dfrac{e^\frac{\hat{q}(S_t, a, \textbf{w}_t)}{\tau}}{\displaystyle\sum_{b\in\cA} e^\frac{\hat{q}(S_t, b, \textbf{w}_t)}{\tau}}, \forall a\in\cA$, where $\tau$ is the temperature.

\subsubsection{Actor-Critic}
Actor-Critic is a policy gradient method that learns an actor and a critic. 
%We used the \emph{One-step Actor-Critic (episodic)} from the introductory text \citep{sutton2018reinforcement} with SGD optimizer. We used a softmax policy actor ($\pi(a \mid S_t) = \dfrac{e^{p(a\|S_t)}}{\displaystyle\sum_{b\in\cA} e^{p(b\|S_t)}}, \forall a\in \mathcal{A}$, where $p$ indicating the preference of selecting the action) and a linearly parameterized critic. 
We used the \emph{One-step Actor-Critic (episodic)} from the introductory text \citep{sutton2018reinforcement} with SGD optimizer. We used a softmax policy actor
%  ($\pi(a \mid S_t) = \dfrac{e^{q(S_t, a)}}{\displaystyle\sum_{b\in\cA} e^{q(S_t, b)}}, \forall a\in \mathcal{A}$, where $p$ indicating the preference of selecting the action) 
and a linearly parameterized critic. 
%The actor updates the policy based on the gradient of the policy, while the critic learns a parameterized value function. In our experiment, we use a softmax actor and a linearly parameterized critic. The pseudocode of Actor-Critic is given in Algorithm \ref{actorcritic}.
The actor learns a policy parameterized by $\theta$ that is used for choosing action at each timestep, while the critic learns a state value function parameterized by $\textbf{w}$ to criticize the action online.
At each timestep, the critic is updated regarding the TD-error $\delta = R + \gamma \hat{v}(S', \textbf{w})  - \hat{v}(S, \textbf{w})$, while the actor is updated according to the direction that the critic suggests.

\subsubsection{Fitted-Q Iteration}

%\archit{Explain FQI}
%\textbf{Fitted Q-Iteration}
%\archit{FQI - Mention the algorithm details and the hyperparams we sweep}

\begin{algorithm}[htb]
	\caption{Regularized Fitted Q Iteration}\label{alg:fqi}
	\textbf{Input:} Learning rate $\alpha$, Mini-batch size $n_\text{batch}$, Dataset $\dataset$, Training iterations $K$, Target network synchronization frequency $T$; \\
	Set exponential decay rates for moment estimates $\beta_1=0.9, \beta_2=0.999$, small number $e = 10^{-8}$; Randomly initialize action-value function $\hat{q}_{\textbf{w}}$, parametrized by $\textbf{w}$
	\begin{algorithmic}[1]
		\STATE $\textbf{m} = \textbf{0}$
		\STATE $\textbf{v} = \textbf{0}$
		\FOR {\text{ iteration $k = 1$, \dots, $K$}}
		\IF {$mod(k,T)==0$}
		\STATE Sync target network: $\textbf{w}^{\prime} \gets \textbf{w}$
		\ENDIF
		\STATE Sample random mini-batch $\{(S_i, A_i, R_i,S^{\prime}_{i}, \gamma_i) \}^{n_\text{batch}}_{i=1}$ from $\dataset$
		\STATE (Adam optimizer update:)
		\STATE \hspace*{5mm} Compute the gradient $\textbf{g}$ according to equation \ref{fqi-grad}
		\STATE \hspace*{5mm} $\textbf{m} \gets \beta_1 \textbf{m} + (1 - \beta_1) \textbf{g}$
		\STATE \hspace*{5mm} $\textbf{v} \gets \beta_2 \textbf{v} + (1 - \beta_2) (\textbf{g} \odot \textbf{g})$
		\STATE \hspace*{5mm} $\hat{\textbf{m}} = \textbf{m} / (1 - {\beta_1}^{k})$
		\STATE \hspace*{5mm} $\hat{\textbf{v}} = \textbf{v} / (1 - {\beta_2}^{k})$
		\STATE \hspace*{5mm} $\textbf{w} \gets \textbf{w} - \alpha \hat{\textbf{m}} \odot \dfrac{1}{\sqrt{\hat{\textbf{v}}} + e} $
		\ENDFOR
	\end{algorithmic}
	\textbf{Return:} $\hat{q}_{\textbf{w}}$
\end{algorithm}

Fitted Q Iteration (FQI) is a classic batch RL algorithm. The policy is learned from the offline dataset, then deployed in the real environment, without online learning. This algorithm is included as a baseline, comparing our method with transferring a fixed policy.
At the $k^{\text{th}}$ iteration, we sample transitions $\{ (S_i, A_i, R_i,S^{\prime}_{i}, \gamma_{i} ) \}^{n_\text{batch}}_{i=1}$ from the dataset and Regularized Fitted Q Iteration (RFQI) minimizes the regularized mean squared temporal-difference error (MSTDE) on this mini-batch:
\begin{equation} \label{fqi-loss}
\begin{aligned}
L_k (\textbf{w}) &= \sum_i \Vert y_i - \hat{q}_{\textbf{w}}(S_i, A_i) \Vert^2 + \lambda_f \text{Pen}(\hat{q}_{\textbf{w}})
\end{aligned}
\end{equation}

where $\hat{q}_{\textbf{w}}$ is the action value approximate parameterized by $\textbf{w}$, $y_i = R_i + \gamma_{i} \max_a \hat{q}_{\textbf{w}^{\prime}}(S^{\prime}_{i}, a)$ is the target for transition $i$, $\textbf{w}^{\prime}$ is the fixed target parameters, $\text{Pen}(\hat{q}_{\textbf{w}})$ is a penalty term. and $\lambda_f$ is the regularization coefficient. The gradient of the loss is obtained by differentiation with respect to the weight parameter $\textbf{w}$:

\begin{equation} \label{fqi-grad}
\begin{aligned}
\nabla_{\textbf{w}} L_k (\textbf{w}) =& \sum_i \left[ \left( \hat{q}_{\textbf{w}}(S_i, A_i) - y_i \right) \nabla_{\textbf{w}} \hat{q}_{\textbf{w}}(S_i, A_i) \right] + \lambda_f \nabla_{\textbf{w}} \text{Pen}(\hat{q}_{\textbf{w}}) \\
\end{aligned}
\end{equation}

For action-value approximation we tried both neural networks $\hat{q}^{\text{NN}}_{\textbf{w}}: \mathcal{S}\times\mathcal{A} \rightarrow \mathbb{R}$, and linear function approximation $\hat{q}^{\text{TC}}_{\textbf{w}}(\cdot,\cdot)=\phi^{\intercal} \textbf{w} $, where $\phi: \mathcal{S}\times\mathcal{A} \rightarrow \mathbb{R}^d$ is the tile coding feature mapping. We set the penalty to squared L2 norm of the weights $\text{Pen}(\hat{q}_{\textbf{w}})=\Vert \textbf{w} \Vert^{2}_{2}$. The detailed training algorithm is described in Algorithm \ref{alg:fqi}.

\subsection{Random Selection Baseline}
The random selection baseline simulates the case when we do not know the best hyperparameter setting thus randomly pick one from the list.
In practice, we randomly choose one hyperparameter setting among all hyperparameters for each run. The random seed of each run is as same as the one used in agent learning. Then we check the true performance of the selected hyperparameter learning in the true environment.
%In this experiment, we sweep over 54 hyperparameters of Expected Sarsa) as described in the above section/appendix. And we use 30 different datasets corresponding to 30 different runs. 
% The random strategy baseline thus randomly selects 1 hyperparameter each run (with repetition) using the same seed the corresponding calibration model experiment uses. 

\section{Experiment details}\label{apdx:exp_details}

\subsection{Environments} 

%Talk about the environments we are using (acrobot, puddleworld, cartpole).
%Add the parameters of the environment in the appendix?

We used three environments in our experiments - Acrobot, Puddle World, and Cartpole. 

We used the Puddle World environment with the transition dynamics, reward function, and actions exactly as described in \citep{sutton1996generalization}. Puddle world is an episodic task where the agent starts at a random position in a $1\times1$ area in the environment. The episode ends when the agent reaches the goal region in the upper right corner of the $1\times1$ area. The objective in puddle world is to reach the goal region as fast as possible while avoiding the puddle which gives the agent negative rewards of high magnitude. The deeper the agent gets into the puddle, the lower the reward it gets. Otherwise, the agent gets -1 per step until it gets to the goal. The state-space is 2-dimensional, containing the $x$ and $y$ coordinates. The agent has 4 actions---left, right, up, down.

We used the Acrobot environment with the transition dynamics, reward function, and actions exactly as described in \citep{sutton2018reinforcement}. Acrobot consists of 2 connected links with the top joint fixed and torque applied to the bottom joint. It is an episodic task where the links start in the rest position, pointing downwards. The episode ends when the tip of the bottom link reaches a specific height level. The objective is to raise the tip of the acrobot above the height level as fast as possible. The state space is 6-dimensional, containing the $\cos$ and $\sin$ value of the angle between the first link and the vertical line and the angle between 2 links. The last 2 dimensions are the angular velocities of 2 links. At the beginning of each episode, each dimension of the state is randomly set to be in range -0.1 to 0.1. The reward of each step is -1. The agent is trained to cross the height level with the least number of steps. The agent has 3 actions--- $+1, 0, -1$ torque applied on the joint connecting the two links.

We used the Cartpole environment with the transition dynamics, reward function, and actions exactly as described in \citep{sutton2018reinforcement}. However, we introduce some random gaussian noise with mean = 0 and stddev = 0.1 over the effect of actions. Cartpole consists of a horizontally moving cart and a pole attached on top of it. The cartpole starts with the cart in the centre of the track, and the pole vertical. We use the continuing version of cartpole where the cartpole is transitioned to the start state when the pole falls below some angle or when the cart goes off the track. Note that we do not reset the episode. The objective is to balance the pole vertically by moving the cart left or right. The agent gets a negative reward whenever the pole falls below some angle or when the cart goes off the track. Otherwise, it gets a reward of 0. The state space is 4-dimensional, containing cart position, cart velocity, pole angle, and pole angular velocity. At the beginning of each episode, each dimension of the state is randomly set to be in range -0.05 to 0.05. The agent has 2 actions - left and right.

\subsection{Collection of offline logs of data}
%Talk about evaluation of hyperparams in the calibration model and the selection of the best hyperparam
In real-world systems, we usually have only one dataset to choose the hyperparameters and one chance to deploy the agent. This case corresponds to having 1 random seed in our experiments. To evaluate the calibration model fairly and robustly, we tested it with $30$ random seeds and reported the hyperparmeter selected for all random seeds. Thus there were 30 offline datasets in total, each of them was collected with a different seed. 
%Note that the random seeds used for data collection and parameter selection are different. More experiment details are provided in Appendix \ref{apd:exp_details}.
%Talk about how we consider these simulated environments as real environment by restricted access to them
In our experiments, though we used simulated environments, we treated them as if they are real world environments. We did not assume access to the underlying environment to select hyperparameters for deployment. Instead, we assumed access to offline logs of collected data from these environments. 
%Doing this, the problem was still the same - choosing good hyperparameters from offline logs of data. %From now on, we shall call these environments as the ``real'' environments.

The dataset was collected with a pre-trained fixed policy. For Experiments 1 and 3, we used near-optimal policy to collect 5k transitions. The quality of the policy is determined by the average performance over the latest 1000 steps. In acrobot, the policy is considered as near-optimal when the averaged length of episode is less than 100, while the standard in Puddle World is that the averaged return is larger than -40. During policy learning, once the performance is above the given threshold, we cut-off learning and use the saved policy to collect data. 

In Experiment 2, we investigated the role of data collecting policy and dataset size. In these experiments, we controlled the data quality in a stricter manner. For the data collecting policy experiment, we collected 5k transitions using near-optimal policy, medium policy, and naive policy. We made sure that each dataset had total episodes within a minimum and maximum number. In Acrobot, we require there are at least 50, between 20-30, and between 10-15 episodes in the 5k transitions dataset, for near-optimal, medium, and naive policy respectively. In Puddle World, the number of episodes are more than 200, between 80-100, and between 20-50 respectively. For the dataset size experiment, we used the medium policy to collect 5k, 1k, and 500 transitions. For Acrobot, we made sure the 5k, 1k, and 500 transitions datasets had 20-30, 4-6, and 2-3 episodes respectively. For Puddle World, we made sure they had 80-100, 16-20, and 8-10 episodes respectively.

In Experiment 4, for Cartpole, we collected 10k transitions using a near-optimal policy, a medium policy, and a random policy. The datasets collected by these three policies had 40-50, 80-125, and 400-500 episodes/failures respectively.

In Experiment 5, we used the 500 transitions dataset collected by a medium policy from Experiment 2.

\subsection{Hyperparameter list}
The hyperparameter settings are reported in this section.

\subsubsection{Agents Hyperparameters}

\paragraph{Expected Sarsa}
The Expected Sarsa agent used tile coding features as its input \citep{sutton2018reinforcement}, with 16 tilings and 8 tiles. It used Adam optimizer and optimistic initialization of action values for early exploration. In Adam optimizer, the second momentum was set as 0.999, the other one was swept and reported below. We kept the policy stochastic by using a softmax function over the action values and sampled the action from the induced probability distribution. 

In the acrobot experiments, we did grid search over the following hyperparameters resulting in 54 combinations:
\begin{enumerate}
	\item Adam optimizer learning rate $\alpha$: \{0.003, 0.03, 0.3\}
	\item Adam optimizer momentum $\beta_1$: \{0.0, 0.9\}
	\item Softmax temperature $\tau$: \{1.0, 10.0, 100.0\}
	\item Optimistic weight initialization: \{0.0, 4.0, 8.0\}
	\item Eligibility trace parameter $\lambda_{e}$ = 0.8
\end{enumerate}

In the puddle world experiments, we did grid search over the following hyperparameters resulting in 54 combinations:
\begin{enumerate}
	\item Adam optimizer learning rate $\alpha$: \{0.01, 0.03, 0.1\} 
	\item Adam optimizer $\beta_1$: \{0.0, 0.9\}
	\item Softmax temperature $\tau$: \{1.0, 10.0, 100.0\}
	\item Optimistic weight initialization: \{0.0, 8.0, 16.0\}
	\item Eligibility trace parameter $\lambda_{e}$ = 0.1
\end{enumerate}

In the cartpole experiments, we did grid search over the following hyperparameters resulting in 54 combinations:
\begin{enumerate}
	\item Adam optimizer learning rate $\alpha$: \{0.03, 0.1, 0.3\}
	\item Adam optimizer $\beta_1$: \{0.0, 0.9\}
	\item Softmax temperature $\tau$: \{0.1, 1.0, 10.0\}
	\item Optimistic weight initialization: \{0.0, 6.0, 12.0\}
	\item Eligibility trace parameter $\lambda_{e}$ = 0.023
\end{enumerate}

In Experiment 3, Figure \ref{fig:nonstationary_acrobot}, we learned the policy in the original Acrobot then transfer it to a changed Acrobot with an increased first link.
To learn a good policy (i.e. the number of steps per episode converges and is near-optimal when using this policy) in the orginal environment, we gave the agent enough time (50,000 timesteps) to make sure the performance converges before the agent runs out of the timestep budget. In transferring step, the performance is measured for 15,000 timesteps. The transferred policy was fixed during this period. In calibration hyperparameter transfer, we used the same hyperparameter setting as chosen in experiment 1, and let the agent learn from scratch for 15,000 steps.
In other experiments, we evaluated each policy for 15k timesteps in acrobot, and for 30k timesteps in puddle world. 

For the calibration model, we average the inner loop performance over 10 random seeds, which means the performance in calibration model is averaged over 10 runs given the same dataset. 30 random seeds are tested for the outer loop (thus there are 30 datasets in total). We ensure the agent learns from at least 30 episodes. For example, the timeout setting is 500 when the number of learning steps is 15k.

\paragraph{Actor-Critic}
The Actor-Critic agent used the same tile coding schema as the Expected Sarsa (16 tilings and 8 tiles). The actor uses one function approximator for each action to obtain a list of scores, and the scores of all the actions are converted to probabilities using a softmax function. The critic also uses function approximator to predict the value of a given feature. In our experiments, both the actor and the critic were zero-initialized and used SGD optimizer.

To eliminate some hyperparameter combinations that are less meaningful, we swept actor's learning rate and the ratio between critic's learning rate and actor's learning rate. This is because of the prior knowledge that the actor's learning rate is usually smaller than the critic's learning rate in practice.

In Acrobot and Puddle World experiments, we did grid search over the following hyperparameters resulting in 36 combinations:

\begin{enumerate}
	\item Critic's learning rate $\alpha$: \{0.001, 0.003, 0.01, 0.03, 0.1, 0.3\}
	\item Actor's learning rate: \{0.001$\alpha$, 0.003$\alpha$, 0.01$\alpha$, 0.03$\alpha$, 0.1$\alpha$, 0.3$\alpha$\}
\end{enumerate}

\paragraph{RFQI}
We trained RFQI offline with the same aforementioned 30 datasets. For each training dataset, we chose one of the other offline datasets as the validation set, and did a grid search on the hyperparameter set $\Lambda_{\text{FQI}}$ as described below. After offline training, for each dataset we chose the learned action-value function with the lowest final MSTDE on the validation dataset and deploy an $\epsilon$-greedy policy with respect to this value function to the true environment.
We ran each policy on the true environment for $T_\text{true}$ timesteps. The online run was repeated 30 times with 30 different random seeds. The expected online performance of each deployed policy was calculated as the average performance across all runs.

We used Adam optimizer to perform gradient descent steps. In both acrobot and puddle world experiments, with other hyperparameters in the agent fixed, we did grid search over the following hyperparameters resulting in 30 combinations:
\begin{enumerate}
	\item Adam optimizer learning rate $\alpha$: \{$10^{-1}, 10^{-2}, 10^{-3}, 10^{-4}, 10^{-5}$\}
	\item L2 regularization scale $\lambda_{f}$: \{$10^{-1}, 10^{-3}, 10^{-5}$\}
	\item Neural network hidden layers: \{(64, 64), (128, 128)\}
\end{enumerate}

\subsubsection{Model Construction and Distance Metrics Learning}
KNN calibration model is non-parametric during construction. Transitions are assigned to different trees based on the action at that time step.
NN model and the Laplacian representation model need to be trained though, with batch size 16 and 128 separately.

We used cross validation to pick parameters. 20\% of the transitions in the dataset were randomly set as the test set while the rest of transitions were training set. When measuring the performance of the NN model, we used the mean squared error. A smaller value is considered as a better performance. 
We measured the dynamic awareness (Equation \ref{eq:dynamics}) when evaluating the quality of the Laplacian representation, while a larger value refers to a better performance. 
During training, we tested the performance every epoch (NN model) / every 1k steps (Laplacian representation), if the test performance reduces for 3 consecutive tests, we cut off training. Otherwise, the NN model is trained for at most 100 epochs, and the Laplacian representation is trained for at most 30k steps.

\begin{equation}
\begin{aligned}
\text{Dynamics Awareness}=\tfrac{\sum^N_{i}||\phi_i - \phi_{j\sim U(1,N)}|| - \sum^N_{i} ||\phi_i-\phi'_i||}{\sum^N_{i}||\phi_i - \phi_{j\sim U(1,N)}||}
\end{aligned}
\label{eq:dynamics}
\end{equation}

The chosen parameter for learning Laplacian representations in both environments are as following.
In Acrobot experiment 1, 3, and 5, we swept and chose $\kappa=0.8$, $\beta=5$, $\zeta=0.5$, $\alpha=0.00003$, and the length of trajectory for picking close states was set to 20. In Experiment 2, $\zeta$ was changed to 0.05. In Puddle World experiment 1 and 5, we used $\kappa=0.8$, $\beta=5$, $\zeta=0.05$, $\alpha=0.0003$, and the length of trajectory for picking close states was set to 10. In Experiment 2, the trajectory length was decreased to 5 and $\alpha$ was decreased to 0.0001. For Cartpole experiment, we used $\kappa=0.8$, $\beta=5$, $\zeta=0.05$, $\alpha=0.00003$, and a trajectory with 50 steps.

We tested both using raw states and Laplacian representations as input in NN calibration model. In both cases, the inputs were scaled to [-1, 1] according to the largest and smallest number of the corresponding feature in the dataset. 
The chosen learning rate for training NN model with raw state input and laplacian representaion in Acrobot were both 0.0003. In Puddle World, the learning rates was 0.0001 when using raw state and 0.0003 when using the Laplacian representation.

When picking hyperparameter with the calibration model, we average the performance over 10 inner runs ($n_{run}=10$).

\section{Automatic Hyperparameter Optimization Experiment}
\label{section-cem-experiment-appendix}

In this section we first outline the CEM algorithm and then provide additional experimental details for the experiment using different hyperprameter optimization algorithms. 

\subsection{CEM for Hyperparameter Optimization}\label{app_cem}

\newcommand{\ntop}{N$_\text{top}$}
\begin{algorithm}[htb!]
	\caption{CEM for Hyperparameter Optimization}
	\label{alg:AltCEM}
	\textbf{Input:}
	\textit{lower} and \textit{upper} ranges for each hyperparameter
	
	\textbf{Set:}
	{maxiterations} = $100$, $\alpha = 0.1$ learning rate, tolerance = $10^{-1}$,
	N = 32 number of hyperparameters to sample each iteration,  
	\ntop = $5$ 
	
	\begin{algorithmic}[1]
		\STATE $\mu \gets $ mean of ranges (for truncated multi-variate normal (TMVN))
		\STATE $C \gets $ diagonal matrix with range width on diagonal
		\STATE $\mu_\text{ave} \gets 0$, $\mu_{\text{old-ave}} \gets 2$ tolerance, \quad iteration $\gets 1$
		\WHILE {iteration $<$ {maxiterations} \& $ \|\mu_\text{ave} - \mu_\text{old-ave}\| >$ tol }
		\STATE Sample N hyperparameter settings from TMVN(\textit{$\mu$}, \textit{C}, \textit{lower}, \textit{upper})
		\STATE Evaluate the N hyperparameter settings, using 3 runs
		%\State Take \ntop $\lambda_{i_1}, \ldots, \lambda_{i_{\ntop}}$
		\STATE Compute mean $\mu_{\text{top}}$, covariance $C_{\text{top}}$ of the top N$_\text{top}$ %hyperparameters, using mean $\mu$ to compute covariance $C_{\text{top}}$
		\STATE $\mu \gets (1-\alpha) \mu + \alpha \mu_{\text{top}}$
		\STATE $C \gets (1-\alpha) C + \alpha C_{\text{top}}$
		\STATE $\mu_\text{old-ave} = \mu_\text{ave}$
		\STATE $\mu_\text{ave} = \mu_\text{ave} + \frac{1}{\text{iteration}} (\mu - \mu_\text{ave})$
		\STATE $\text{iteration} \gets \text{iteration} + 1$
		\ENDWHILE
		\STATE return $\mu$
		
	\end{algorithmic}
\end{algorithm}

We initialize a truncated multi-variate normal (TMVN) distribution, for the given ranges for the hyperparameters, at the center of these ranges with a wide initial covariance. 
On each iteration, we 1) sample N = 32 hyperparameter settings from this TMVN, 2) obtain noisy estimates of the performance of these hyperparameter settings using 3 runs $\hat{f}(\lambda_1), \ldots, \hat{f}(\lambda_N)$, 3) select the top 5 (approximately top 10\%) hyperparameter settings, and 4) increase their likelihood under the TMVN. We increase the likelihood by updating the mean and covariance towards the mean and covariance of these top 5 hyperparameter settings. This stochastic update slowly concentrates the mean around high-valued hyperparameters.
%This stochastic update is an approximate incremental update 

We run the optimization for a maximum of 100 iterations. However, we also allow for early stopping if fewer iterations are required. If the mean has stopped changing earlier than 100 iterations, then we can stop the optimization early. Using the change in mean, however, might stop too early, especially in early training, due to the noise in performance estimates. Instead, we use a long-run average and compare the change in this average, to obtain a stopping condition. We still return $\mu$ rather than the long-run average, because $\mu$ reflects the mean we have concentrated around. 

This algorithm is designed for continuous hyperparameters. We can, however, also apply it to certain discrete hyperparameters.
The discrete hyperparameters are of two types: ordered and unordered. 
For unordered discrete hyperparameters, the CEM procedure should simply be run for each of the discrete hyperparameter settings, with CEM picking amongst the other ordered hyperparameters. This is because CEM is faster than grid search when there is the ability to generalize between hyperparameters. Without order, we will not have generalization.

For ordered discrete hyperparameters, such as the number of tilings, we optimize them as continuous hyperparameters, but round to the nearest integer. For example, the range for number of tilings 1, 2, 4 and 8 could be converted to $h \in [0,3]$ where $2^h$ is the number of tilings and when we sample $h \in [0,3]$ we round to an integer before querying for the performance of that hyperparameter. We can see this as the agent taking in the hyperparameter $h \in [0,3]$ and itself doing the rounding and exponentiation as part of the agent. This transformation allows us to use generalization to reason about if the agent prefers a smaller or bigger number of tilings. 

%This approach allows for the simplicity of using a truncated Gaussian distribution, specifially a truncated multivariate normal (TMVN) distribution. This multivariate normal allows us to maintain a covariance, which means we can more quickly discard joint hyperparameter settings that are poor and reduce our search space. 
%
%One other nuance is that we do not get a perfect sample of the performance $f(h)$ of a hyperparameter. Rather, we get a noisy estimate $\hat{f}(h)$ because we do a limited number of runs per hyperparameter to save on computation. Instead, the approach is to reason across CEM iterations about what region of hyperparameters is effective. To do so, we use an incremental update to the parameters of the TMVN, similar to a stochastic gradient descent update. For a sufficiently small learning rate, the distribution will concentrate and will converge to a set of hyperparameters. We use a learning rate of $0.1$, since across experiments, we found this to be effective. 
%
%Finally, we use Polyak averaging, as a way to overcome the noise in iterates in SGD, especially for a fixed learning rate. We maintain a running average of the SGD iterates. When this average stabilizes, we consider the algorithm to have converged. 

\subsection{Experimental Details for Automatic Hyperparameter Optimization}
In Experiment 5, when CEM was used, we tuned the temperature $\tau$ and learning rate $\alpha$ of the Sarsa agent as continuous values in the ranges $[0.0001, 5.0]$ and $(0.0, 0.1]$ respectively for Acrobot, and $[0.0001, 10.0]$ and $(0.0, 1.0]$ respectively for Puddle World. We ran the CEM experiment once on each of the 30 datasets, to finally get 30 hyperparameters tuned by CEM. The CEM box plot in Figure \ref{fig:exp_cem} corresponds to the true performance of these 30 hyperparameters in the real environment. We ran CEM experiment for 100 iterations in Acrobot, and 30 iterations in Puddle World. Each iteration samples 32 hyperparameters which were all evaluated for 5 runs to get the performance measure of each hyperparameter within that iteration. 

We used an open-source package for the Bayesian optimization \citep{bayesopt}, it takes the gaussian process for optimizing the hyperparameter setting. The datasets and hyperparameter ranges were kept the same as in CEM experiment. For the optimization strategy itself, we chose Upper Confidence Bounds as the acquisition method and ran for 200 iterations. The confidence value in Upper Confidence Bounds, which controls the level of exploration, was set to 2.576. The queue is initialized with 5 random samples.

The best random search baseline shows the best performance among all uniformly sampled hyperparameters. The number of sampling was kept the same as the number of iterations. The range of hyperparameter was the same as the range used by CEM and Bayesian optimization.

We also compared CEM with grid search. Both used the same calibration model.
To compare against the CEM results in Acrobot (LHS in Figure \ref{fig:exp_cem}), we did a grid search over the following hyperparameters:
\begin{enumerate}
	\item Adam optimizer learning rate $\alpha$: \{0.001, 0.003, 0.01, 0.03, 0.1\}
	\item Softmax temperature $\tau$: \{0.0001, 1, 2, 3, 4, 5\}
\end{enumerate}

We kept the following hyperparameters fixed
\begin{enumerate}
	\item Adam optimizer momentum $\beta_1$ = 0.0
	\item Optimistic weight initialization = 0.0
	\item Eligibility trace parameter $\lambda_{e}$ = 0.8
\end{enumerate} 

To compare against the CEM results in Puddle World (RHS in Figure \ref{fig:exp_cem}), we did a grid search over the following hyperparameters:
\begin{enumerate}
	\item Adam optimizer learning rate $\alpha$: \{0.001, 0.003, 0.01, 0.03, 0.1\}
	\item Softmax temperature $\tau$: \{1, 2, 4, 6, 8, 10\}
\end{enumerate}

We kept the following hyperparameters fixed
\begin{enumerate}
	\item Adam optimizer momentum $\beta_1$ = 0.0
	\item Optimistic weight initialization = 0.0
	\item Eligibility trace parameter $\lambda_{e}$ = 0.1
\end{enumerate}

%% Han: Moved to experiment 6
%In Figure \ref{fig:cem_acrobot}, we report the results of using CEM in acrobot. We tuned the temperature $\tau$ and learning rate $\alpha$ of the Sarsa agent as continuous values in the ranges [0.0001, 5.0] and (0.0, 0.1] respectively for Acrobot. 

\section{Additional Experimental Results} \label{apdx:more_results}
Additional experiments results will be provided in this section. Including the result when removing the Laplacian representation distance metric (Section \ref{apdx:calibration_raw}), experiment result for actor-critic on Acrobot and Puddle World (Section \ref{apdx:actor_critic}), additional discussion on the failure case (Section \ref{apdx:cartpole}), and CEM performance on Acrobot (Section \ref{apdx:cem_acrobot}).

\subsection{Calibration Model with Raw States}\label{apdx:calibration_raw}

\begin{figure}[t]
	%	\vspace{-0.3cm}
	\centering
	\includegraphics[width=.8\linewidth]{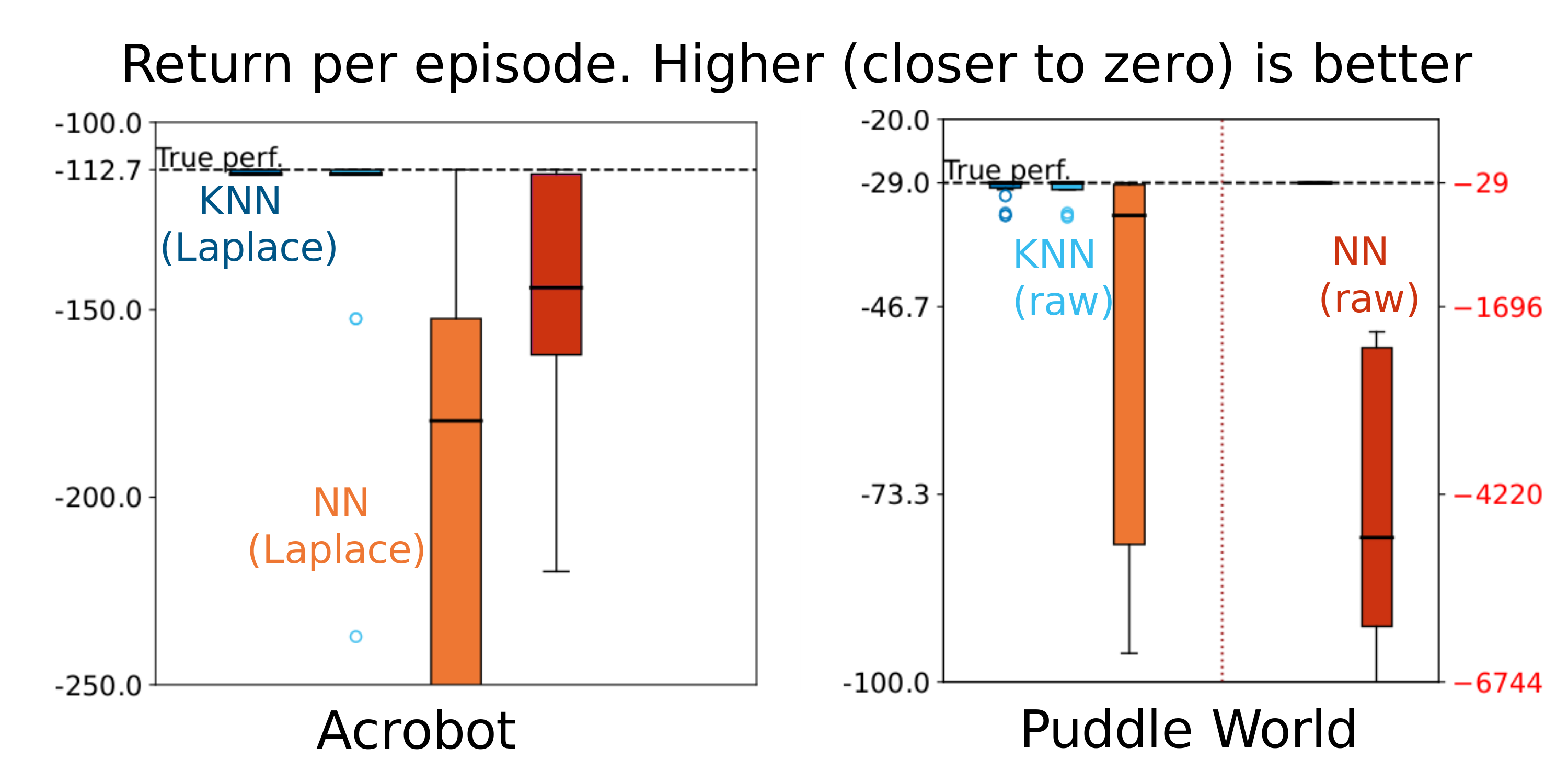}
	%	\vspace{-0.3cm}
	\caption{\textbf{Replacing Laplacian representation with raw state.} Each subplot shows the result of comparing using raw state and Laplacian representation in KNN based and NN based calibration model. In both subplots, the \textbf{higher return is better}.} 
	\label{fig:calibration_raw}
	%	\vspace{-0.4cm}
\end{figure}

Using the Laplacian representation gives the kd-tree a more reliable distance metric to search for the nearest neighbor regarding the transition dynamics. We also empirically checked the calibration model performance when removing the Laplacian representation and using the L2 distance on raw states as the distance metrics. It turns out there exist more outliers in this case (Figure \ref{fig:calibration_raw}).
Regarding NN based calibration model, the one using Laplacian representation input performs better on Puddle World, while the other performs better on Acrobot. In both cases, the KNN based calibration model performs better than the NN-based calibration model.

\subsection{Actor-Critic Agent on Acrobot and Puddle World} \label{apdx:actor_critic}

\begin{figure}[t]
	\centering
	%	\vspace{0.3cm}
	\includegraphics[width=.8\linewidth]{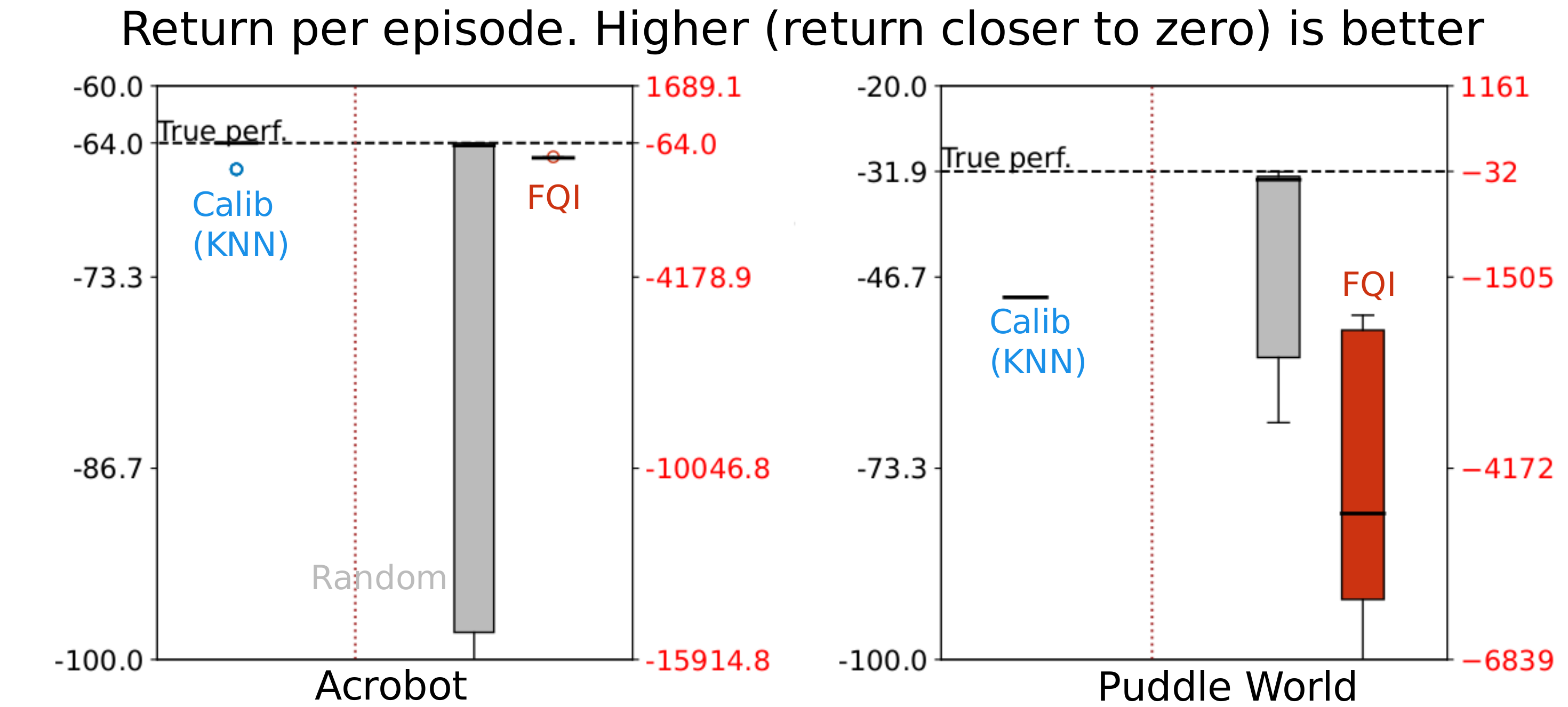}
	\caption{\textbf{Calibration model hyperparameter transfer with Actor-Critic agent.} Each subplot shows the calibration model performance with Actor-Critic agent, compared against FQI and a random baseline. The dotted horizontal line indicates the performance of Actor-Critic agent with the best hyperparameter setting in the sweep in the deployment environment. In both subplots, the \textbf{higher return is better}.} 
	\label{fig:actor_critic}
	%	\vspace{-0.4cm}
\end{figure}
Two agents were tested in our experiment, Expected Sarsa and Actor-Critic. This section provides calibration model performance when using Actor-Critic agent on Acrobot and Puddle World.
The chosen parameter performance has a relatively larger difference than Expected Sarsa, but the hyperparameter transfer with calibration model still works better than the random baseline and FQI policy transfer. The performance is given in Figure \ref{fig:actor_critic}.

\subsection{Additional Discussion on Cartpole}\label{apdx:cartpole}

In Experiment 4, we used Noisy Cartpole as a continuing learning problem. Thus, when we collected a dataset with a near-optimal policy, there exist 2 problems: (1) most of the transition has reward=0, and (2) the dataset almost does not include the failure case that the agent usually sees at the early-learning stage.
There are 2 failure types in Cartpole: one is when the pole drops (\emph{angular failure}), and the other is the cart goes too far and out of the range (\emph{positional failure}).
The failures under the near-optimal policy are mostly positional failures. However, the failure that the agent sees most often at the early learning stage is the angular failure. So with the missing failure case, it is hard for the agent to obtain useful reward information when learning with the calibration model trained with this dataset.

\subsection{Fitted-Q Iteration with Neural Network on Acrobot and Puddle World}\label{apdx:cem_acrobot}

\begin{figure*}
	%	\vspace{-0.3cm}
	\centering
	\includegraphics[width=.8\linewidth]{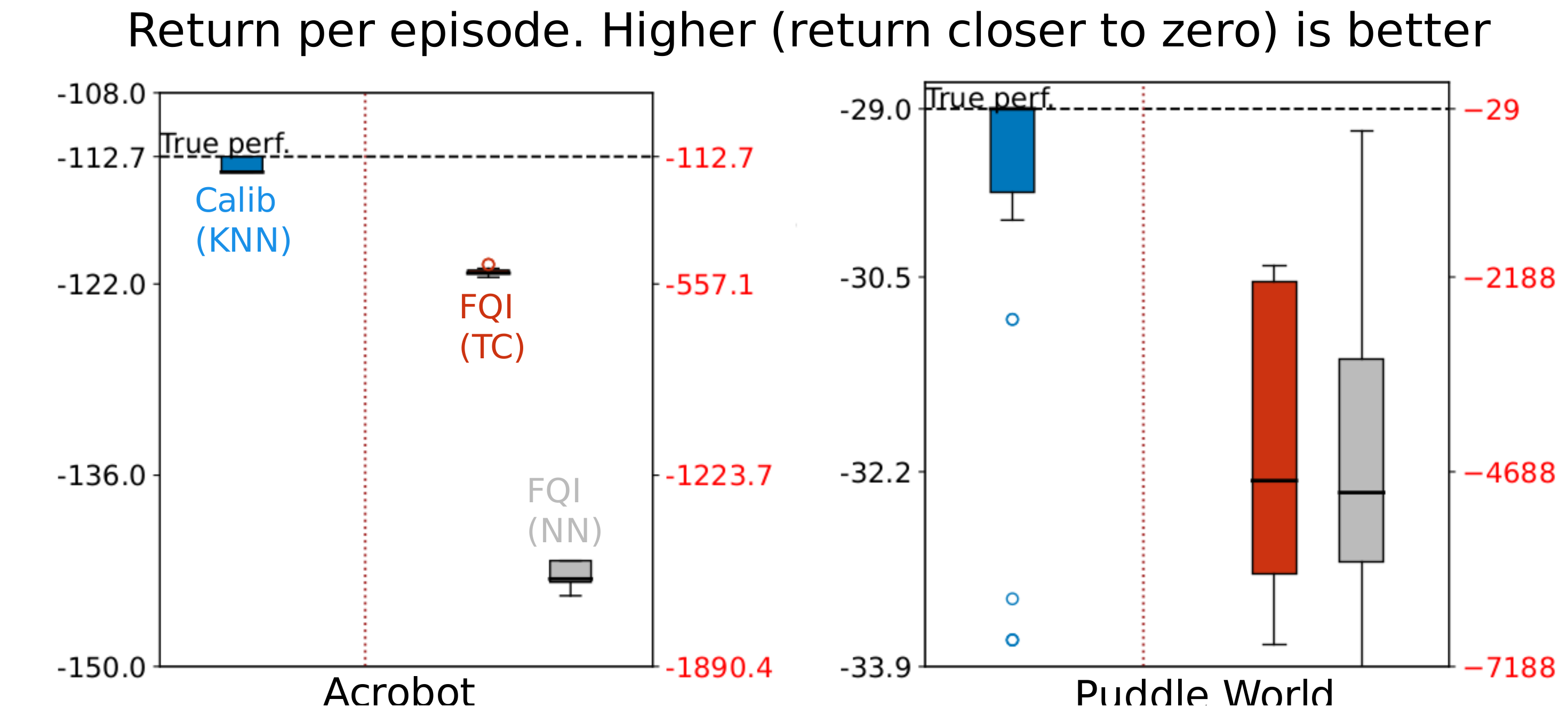}
	\caption{\textbf{Performance of FQI with linear and nonlinear function approximators.}} 
	\label{fig:fqi_nn}
	%	\vspace{-0.4cm}
\end{figure*}

We tested 2 settings in Fitted-Q Iteration, using tile coding with linear function approximator and using raw state with nonlinear function approximator. We report the first setting in Section \ref{exp_sec} as it is consistent with the setting in Expected Sarsa. When comparing both settings, the linear function approximator performs better in Acrobot, while the nonlinear function approximator performs better in Puddle World. However, when comparing them with the KNN based calibration model as baselines, both settings perform worse. The plots are shown in Figure \ref{fig:fqi_nn}.

%% Han: Moved to Experiment 6
%\subsection{CEM on Acrobot}
%
%\begin{figure*}
%	\centering
%	\includegraphics[width=.4\linewidth]{figapp_cem.pdf}
%	\vspace{-0.5cm}
%	
%	\caption{\textbf{Combining calibration model with CEM in Acrobot.} The performance of hyperparameter chosen by CEM in calibration model compared with hyperparameter sweeping in the calibration model. The y-axis is same as the one in Figure  \ref{fig:model_works}.} 
%	\label{fig:cem_acrobot}
%	
%\end{figure*}
%We add the result for combining calibration model with CEM in Acrobot in Figure \ref{fig:cem_acrobot}.
%CEM outperforms grid search in Acrobot as well, showing a higher chance of getting a hyperparameter that shows closer or higher performance to the true best performance in grid search. The hyperparameters used in this experiment are mentioned in section \ref{section-cem-experiment-appendix}
%

\end{document}